\documentclass[10pt,conference]{IEEEtran}
\usepackage{cite}
\usepackage{amsmath,amssymb,amsfonts,amsthm,mathtools}
\usepackage[noend]{algorithmic}
\usepackage{graphicx}
\usepackage{textcomp}
\usepackage{xcolor}
\usepackage{hyperref}
\def\BibTeX{{\rm B\kern-.05em{\sc i\kern-.025em b}\kern-.08em
    T\kern-.1667em\lower.7ex\hbox{E}\kern-.125emX}}
\usepackage{ifthen}
\usepackage[normalem]{ulem} 
\usepackage{algorithm}

\DeclareMathOperator*{\maj}{\mathrm{Maj}}

\newtheorem{definition}{Definition}[section]
\newtheorem{theorem}{Theorem}[section]

\newtheorem{lemma}[theorem]{Lemma}

\newboolean{showedits}
\setboolean{showedits}{true} 
\ifthenelse{\boolean{showedits}}
{
	\newcommand{\del}[1]{\textcolor{red}{\sout{#1}}} 
}{
	\newcommand{\del}[1]{} 
	
}

\newboolean{showcomments}
\setboolean{showcomments}{true} 
\newcommand{\id}[1]{$-$Id: scgPaper.tex 32478 2010-04-29 09:11:32Z oscar $-$}

\ifthenelse{\boolean{showcomments}}
{\newcommand{\nbc}[3]{
 {\colorbox{#3}{\bfseries\sffamily\scriptsize\textcolor{white}{#1}}}
 {\textcolor{#3}{\sf\small$\blacktriangleright$\textit{#2}$\blacktriangleleft$}}}
 }
{\newcommand{\nbc}[3]{}
 \renewcommand{\del}[1]{} 
 }

\definecolor{ibcolor}{rgb}{0.9,0.5,0}
\definecolor{accolor}{rgb}{0,0.5,0.9}
\definecolor{zjcolor}{rgb}{0,0.2,0.1}
\definecolor{smcolor}{rgb}{1.0,0.1,0.1}
\definecolor{tdcolor}{rgb}{1.0,0,0}

\begin{document}

\title{Fairness-guided SMT-based Rectification of Decision Trees and Random Forests}

\author{\IEEEauthorblockN{Jiang Zhang}
\IEEEauthorblockA{
\textit{National University of Singapore}\\
zhangj@comp.nus.edu.sg}
\and
\IEEEauthorblockN{Ivan Beschastnikh}
\IEEEauthorblockA{
\textit{University of British Columbia}\\
bestchai@cs.ubc.ca}\\
\IEEEauthorblockN{Abhik Roychoudhury}
\IEEEauthorblockA{
\textit{National University of Singapore}\\
abhik@comp.nus.edu.sg}
\and
\IEEEauthorblockN{Sergey Mechtaev}
\IEEEauthorblockA{
\textit{University College London}\\
s.mechtaev@ucl.ac.uk}
}

\maketitle
\thispagestyle{plain}
\pagestyle{plain}

\begin{abstract}
Data-driven decision making is gaining prominence with the popularity of various machine learning models. Unfortunately, real-life data used in machine learning training may capture human biases, and as a result the learned models may lead to unfair decision making. In this paper, we provide a solution to this problem for decision trees and random forests. Our approach converts any decision tree or random forest into a fair one with respect to a specific data set, fairness criteria, and sensitive attributes. The \emph{FairRepair} tool, built based on our approach, is inspired by automated program repair techniques for traditional programs. It uses an SMT solver to decide which paths in the decision tree could have their outcomes flipped to improve the fairness of the model. Our experiments on the well-known adult dataset from UC Irvine demonstrate that FairRepair scales to realistic decision trees and random forests. Furthermore, FairRepair provides formal guarantees about soundness and completeness of finding a repair. Since our fairness-guided repair technique repairs decision trees and random forests obtained from a given (unfair) data-set, it can help to identify and rectify biases in decision-making in an organisation.
\end{abstract}


\section{Introduction}

Due to the advent of data-driven decision-making, many important societal decisions are increasingly being made by programs representing learning models.  It is increasingly common to have decision making for resource allocation, such as approval of bank loans, to be conducted by programs representing learning models. Such programs are constructed based on training data. So, for approval of loans, the training data contain relevant information about past loan applications and approval decisions. Unfortunately, past decisions could be influenced by human biases. Therefore, the training data could be biased, and the decision making programs obtained from such data could be unfair.

Testing, analysis, and verification of decision-making programs in relation to fairness properties have been studied in the last five years. However, previous work has not studied repairing decision trees and random forests which may be unfair. We focus on these two models, since they are commonly used for capturing decision making in software systems in various fields including industrial operations research.
The notion of fairness that we adopt in this paper combines the \emph{group fairness} criteria used by Galhotra et al. \cite{FairnessTestingESECFSE2017} and Albarghouthi et al. \cite{CAV17}. In simple words, group fairness requires the ratio of the probability of a minority group member getting a resource to the probability of a majority group member getting the resource, to be bounded from both below and above.  


\begin{figure}
    \centering
    \includegraphics[width=\columnwidth]{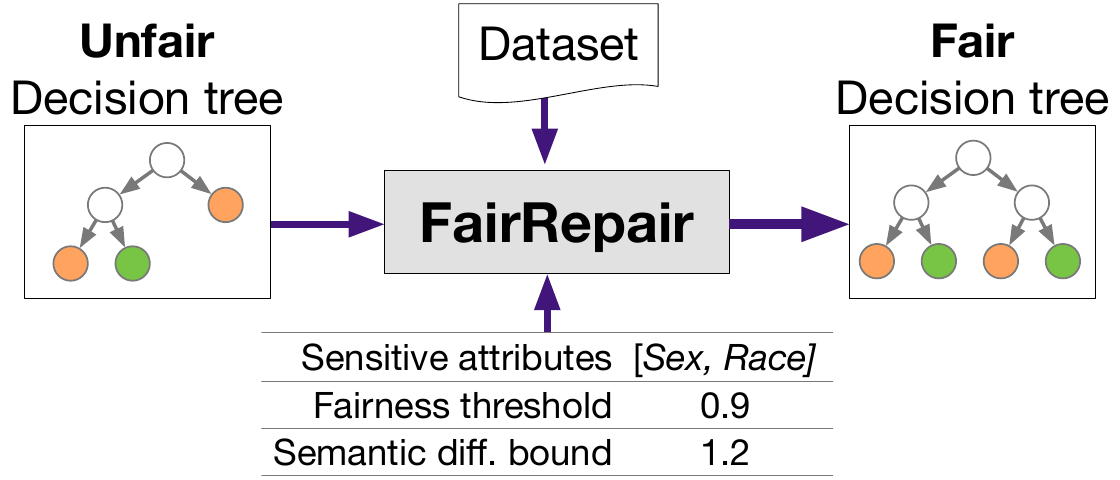}
    \caption{Our technique repairs an unfair decision tree into a fair one relative to a dataset and fairness requirements.
    \label{fig:fair-repair}}
\end{figure}{}

In this paper, we study the task of automatically repairing decision-making programs, possibly derived from unfair training data, to achieve fairness. To the best of our knowledge, the only existing approach for fairness guided repair is DIGITS \cite{CAV17}. DIGITS relies on probability distributions to represent inputs, and to give (probabilistic) guarantees of completeness. In practice, it is non-trivial to compute input distributions. Unlike DIGITS, our technique relies on a set of input-output pairs.  We assume that a (possibly unfair) dataset of past human-made decisions is available. Instead of directly returning a classifier trained on this dataset, our technique generates a fair classifier by repairing an unfair classifier (Figure~\ref{fig:fair-repair}). 

Our \emph{FairRepair} approach is intended to support different use-cases than a tool like DIGITS \cite{CAV17}. The DIGITS approach can potentially help in formulating high level policies for fair decision making in a population based on a \emph{known} population distribution. In contrast, FairRepair can identify and rectify unfairness in decision making based on past decisions' data; the population distribution may be unknown. We elaborate on the distinction between the two methodologies in Section \ref{related}. The output from FairRepair can be used in at least two ways: (1) the repairs may help the users understand implicit or unconscious biases in manual decisions, and can explain how to improve organizational decision-making workflows, or (2) FairRepair can help to migrate from manual decision making to a fair AI-based process by training classifiers on historical decision data and correcting them using fairness requirements.

\paragraph*{Contributions}
The main contribution of this paper is a solution to fairness-guided automated repair of decision trees and random forests. Our approach automatically repairs a given unfair decision tree (or random forest), and minimizes the semantic difference between the original and the repaired tree (or forest) by relying on a partial MaxSMT solver. Our technique is \textit{sound} and \textit{complete}, i.e., it will find a repair if one exists and the repair is guaranteed to satisfy the fairness semantic difference requirements. 

We evaluate FairRepair on two well-known publicly available datasets (German data-set and  Adult data-set from UC Irvine). We show that our implementation achieves the claimed fairness guarantees for different fairness thresholds and sensitive attributes: (1) FairRepair outputs decision trees that satisfy group fairness, and (2) FairRepair adheres to a bound on the number of input points whose classification result differs in the output decision tree or random forest model. We investigate the classification accuracy of our repaired models, and observe that our repairs preserve the accuracy of the original model. 
We also study FairRepair's scalability and find that it is able to find a repair for decision trees on a dataset with 48,842 data points in under 5 minutes (average time). 
We make our tool, {\em FairRepair}\footnote{\url{https://github.com/fairrepair/fair-repair}}, available online, and we plan to make it open-source.

\section{Overview}

Consider the problem of approving loan applications. The decision is either $\mathtt{TRUE}$ (approve application) or $\mathtt{FALSE}$ (reject application). Suppose we know that the existing decision-making program is unfair (possibly due to being derived from an unfair data-set) for some pre-defined fairness requirements. Our goal is to produce a repair that is also a decision-making program, but one that is fair and with bounded semantic difference from the unfair program.

Let $D$ be a dataset of applications which includes three attributes: an applicant's sex, education level, and age, which are {\em discrete} attributes. Examples of possible attribute values are listed below:
\begin{center}
\begin{tabular}{ |c|c|c| } 
 \hline
 \textit{\textbf{sex}} & education & \textit{age} \\ 
 \hline
 \textit{male, female} & \textit{low, med, high} & $\{1,\dots, 99, 100^+\}$  \\
 \hline
\end{tabular}
\end{center}

Let \textit{sex} be a \emph{sensitive} attribute.
Each sensitive attribute partitions the input space into disjoint subsets, which we call the \textit{sensitive groups}. Continuous attributes could also be sensitive, which would create partitions in the form of disjoint intervals over the value range. For multiple sensitive attributes, sensitive groups are the cross product of the partitions for each sensitive attribute. In our example, fairness requirements are probabilistic inequalities of proportions of applicants in a sensitive group that receive loans. We call these proportions \textit{passing rates}. In our example, the fairness requirements are defined as follows:
\[
\frac{\mathbb{P}(\text{get loan}|\text{female})}{\mathbb{P}(\text{get loan}|\text{male})}>0.8 \text{ and } \frac{\mathbb{P}(\text{get loan}|\text{male})}{\mathbb{P}(\text{get loan}|\text{female})}>0.8.
\]

The above inequalities restrict the ratios of the proportions of applicants receiving loans among males and females to be bounded by a fairness threshold $c$, 0.8. In general, fairness thresholds for each inequality are not required to be the same. For illustration purpose, we will use 0.8 for both inequalities.

The input distribution on the non-sensitive attributes may vary among different sensitive groups. Let $\mathit{edu_1}$ and $\mathit{edu_2}$ be two distinct valuations of education. Suppose 80\% of male applicants have $\mathit{edu_1}$ and 80\% of female applicants have $\mathit{edu_2}$. If the decision making procedure uses only education to decide whom to give loans, to meet the fairness requirements, either both groups of applicants receive loans, or both groups do not receive loans. In general, when the decision making process only uses the non-sensitive attributes, the resulting program may experience over-fitting. Hence, to achieve group fairness, it is reasonable to set disparate decision making schemes for each sensitive group, in order to bound the ratios of the passing rates. We next explain how we repair an unfair decision tree.

\paragraph{Pre-processing}
Let $T$ be an unfair decision tree trained using a dataset $D$. The first step of our approach is to collect paths with respect to the sensitive groups. Each path is interpreted as a region (or a \textit{hypercube}) in the input space with a coloring (its return value). In our example, there are two colorings, $\mathtt{TRUE}$ and $\mathtt{FALSE}$.
To avoid ambiguity, we shall use the phrase \textit{path hypercube} throughout the paper.  An example path is a constraint of the form $\text{(sex=female)}\wedge\text{(age}>\text{50)}\wedge\text{(edu=high)}\to\mathtt{TRUE}$. It can be naturally represented as a hypercube: $\{\text{female}\}\times (50,\infty)\times\{\text{high}\}$. 
To account for sensitive attributes explicitly, we collect the intersections of these path hypercubes with the sensitive groups. Each path hypercube will be split into disjoint subsets with the same coloring inherited. The constraints of a path hypercube may or may not contain sensitive attributes. The above-mentioned path hypercube will have an empty intersection with sensitive group \textit{male}. All such subsets of path hypercubes will be collected based on their sensitive groups, and form two sets, $S_m$ for male and $S_f$ for female. 

\paragraph{Calculating Lower Bound of Minimal Change}
For each path hypercube $i$, we denote its return value as $I_i$, and calculate its path probability $p_i$ (by counting the number datapoints residing in it), and passing rate $r_i$ by considering the data points that receive the desired outcome (getting a loan). Similarly, we compute the proportions of each sensitive group in the entire population. We use linear optimisation to calculate the minimum change in these proportions to satisfy the fairness requirement. Let the path probability and passing rate for male (resp. female) be $p_{\textit{m}}$ and $r_{\textit{m}}$ (resp $p_{\textit{f}}$ and $r_{\textit{f}}$). Let $x_m$ and $x_f$ be \emph{real} variables for passing rates of male and female. We are to solve
\[
\begin{cases}
\dfrac{x_m}{x_f}\geq 0.8, \dfrac{x_f}{x_m} \geq 0.8 \text{ and,}\\
\text{minimise } p_m\cdot|x_m-r_m|+p_f\cdot|x_f-r_f|.
\end{cases}
\]
These minimum changes are used to calculate the theoretical lower bound of semantic difference $\mathit{sd}$, i.e., the proportion of the population receiving different outcomes before and after the change. However, the lower bound might not always be achievable in practice. For \emph{real} solutions of $x_m$ and $x_f$, the minimal changes computed in the numbers of data points for males and females may be non-integers, but these numbers must be integers as we are counting people. Nonetheless, $sd$ serves as a good starting point for our search for a repair.

\paragraph{Calculating Patches with an SMT Solver}
With the minimal theoretical semantic difference $\mathit{sd}$, we encode the fairness requirements as SMT formulas, and encode the return values of the path hypercube subsets in $S_m$ and $S_f$ as variables. We first check if it is possible to repair $T$ without modifying its tree structure, i.e., by only flipping the return values of some path hypercubes. We assign a pseudo-Boolean variable $X_i$ for each path hypercube. For our example, the fairness requirements are encoded as the following hard constraints:
\[
\begin{cases}
\sum\limits_{i\in S_m}X_i\cdot p_i / p_m \geq 0.8 \sum\limits_{i\in S_f}X_i\cdot p_i / p_f \text{ and,}\\
\sum\limits_{i\in S_f}X_i\cdot p_i / p_f \geq 0.8 \sum\limits_{i\in S_m}X_i\cdot p_i/p_m.\\
\end{cases}
\]
Since $sd$ might not be achieved, yet we want the repair to have a small semantic difference, we adopt a multiplicative factor $\alpha$ as follows ($I_i$ is the return value of path hypercube $i$):
\[
\sum\limits_{i \in S_m}p_i\cdot\mathrm{Xor}(X_i ,I_i) + \sum\limits_{j \in S_m}p_i\cdot\mathrm{Xor}(X_i,I_j) \leq \alpha \cdot\mathit{sd}.
\]
The left hand side is the actual semantic difference, and $\alpha$ is used to bound its difference with $sd$. If the solver returns $\mathtt{UNSAT}$, we refine the path hypercubes by splitting them on an attribute value, and re-run the query. We refine repeatedly until we find a repair.

\section{Preliminaries}
\label{subsec: defs}

We now formally define the \textit{fairness repair} problem.

\paragraph{Attributes} An attribute $A$ is a label that takes some value(s) from its attribute space. An attribute can be discrete or continuous. If $A$ is discrete, the attribute space of $A$ is a finite unordered discrete space. Examples of discrete attributes are gender, education level, and occupation. If $A$ is continuous, the attribute space of $A$ is a one-dimensional real space, or $\mathbb{R}$. Examples of continuous attributes are height and salary. For simplicity, we use $A$ to denote both the attribute and its attribute space when no ambiguity arises.

\paragraph{Input Space and Input Distribution} An input space $S$ is the Cartesian product of one or more attribute spaces. We call an input space $n$-dimensional if its total number of attribute spaces is $n$. The input distribution $p$ is a probability density function defined on the input space. It is the distribution of the \textit{population}. An input space is not necessarily associated with an input distribution. An element in the input space is called a \textit{candidate}. 

\paragraph{Data Set} A data set $D$ is a finite set of data points, associated with an $n$-dimensional input space $S$. Each data point contains an $n$-tuple $(a_1,a_2,\dots,a_n) \in S$ and an outcome $b$. The outcome $b$ is a boolean value, i.e., $\mathtt{TRUE}$ or $\mathtt{FALSE}$. If the input space is not associated with an input distribution, we assume the data set represents the input distribution perfectly. The frequencies of $n$-tuples in the data set are interpreted to be the density function $p$.

\paragraph{Sensitive Attributes and Sensitive Groups} Each input space has a set of attributes. Some attributes are considered \textit{sensitive}, e.g., sex and race. A sensitive attribute's attribute space is partitioned into a finite number of disjoint subsets. For discrete attributes, $A$ is partitioned into $|A|$ subsets, where each partition contains exactly one element of $A$. For continuous attributes, $A$ is partitioned based on some thresholds. These thresholds divides $A$ into finitely many intervals whose union is $A$. Sensitive groups $S_i$'s are subsets of the input space, partitioned by combinations of different values of sensitive attributes. For sensitive attributes $\{A_1,\dots,A_m\}$, the input space is partitioned into $M:=\prod_m|A_i|$ sensitive groups. 

\paragraph{Decision Making Program and Passing Rates} A decision making program $P$ is defined on an input space $S$ and a data set $D$, such that $\forall x \in D. P(x)\in\{\mathtt{TRUE},\mathtt{FALSE}\}$. For program $P$, passing rate $r$ of a sensitive group $S_i$ is the probability of a candidate receiving a pre-defined outcome conditioned on being in the sensitive group $S_i$. If not specified, the pre-defined outcome is $\mathtt{TRUE}$. Formally, $r:=\mathbb{P}(P(x) = \mathtt{B}|x\in S_i)$, where $\mathtt{B} = \mathtt{TRUE}$.

\paragraph{Group Fairness} For an input space and a decision making program, if for every two sensitive groups, the difference in their passing rates is bounded by a multiplicative factor, $0<c <1$, then we say that the decision making program attains group fairness. This definition is adapted from \cite{FairnessTestingESECFSE2017}. Formally, we require 
\[
\forall 1\leq i,j \leq M,  c\,r_i \leq r_j \leq \frac{r_i}{c} .
\]

\paragraph{Semantic Difference} For an input space $S$ and two decision making programs $\mathit{Prog}_1$, $\mathit{Prog}_2$ defined on $S$, their semantic difference $SD(\mathit{Prog}_1, \mathit{Prog}_2)$ is defined as the proportion in the population that receives different outcomes in the two programs. Formally, 
\[
SD(\mathit{Prog_1},\mathit{Prog}) :=\mathbb{P}(\mathit{Prog_1}(x) \neq \mathit{Prog_2}(x) | x \in S)
\]

\begin{definition}[Fairness Repair Problem]
Given an input space $\emph{S}$ and a decision making program $\emph{Prog}$ that does not satisfy group fairness, a $\emph{repair}$ is a decision making program that attains group fairness. An $\emph{optimal repair}$ is a repair that minimises the semantic difference between itself and $\emph{Prog}$. 

To solve a $\emph{fairness repair problem}$ is to find a repair $R$ with a small semantic difference with $\emph{Prog}$, i.e., if there exists an optimal repair $R_{op}$, $R$ should be bounded by a multiplicative factor $\alpha >1$ compared to that of the optimal repair. Formally, $SD(R, \emph{Prog}) \leq \alpha\cdot SD(R_{op}, \emph{Prog})$.
\end{definition}


\paragraph{Optimal Repair}
\label{sec: optimal}
We now discuss the existence of a theoretical optimal solution for decision tree repair. In this context, an \textit{optimal solution} is a decision tree $T'$ that satisfies the fairness requirement and has a minimal semantic difference compared to the original decision tree $T$. 

We shall still use $M$ to denote the number of sensitive groups. These sensitive groups are shared by $T$ and $T'$. Let $T(d)$ denote the predicted outcome for $d\in D$. We abuse the notation to let $S_i, 1\leq i \leq M$ also denote the set of path hypercubes in the sensitive group $S_i$. Let $P_{i,j}$'s and $p_{i,j}$'s denote the path hypercubes and path probabilities, where $1\leq i \leq M$, $1\leq j \leq |S_i|$. For a path hypercube $P_{i,j}$, we use $X_{i,j}$ to denote desired return value (variable) of the path hypercube. Each $X_{i,j}$ is a pseudo-Boolean variable. Let $\mathcal{M}:D\to \{X_{i,j}:1\leq i \leq M,1\leq j \leq |S_i|\}$ be the function that maps the datapoint to the pseudo-Boolean variable which corresponds to the path hypercube that the point resides in. Since the data set (by assumption) correctly represents the input distribution, if there exists a mapping of data points to return values, and it satisfies the fairness requirement, then such a mapping is by definition a repair. In addition, we would like to minimise the semantic difference. This is equivalent to minimise the number of data points that have different return values before and after the repair, and can be formulated as SMT formulas and directly solved by the solver, as follows:
\[
\begin{cases}
\forall\, 1\leq i, j\leq M, k, l, \frac{\sum\limits_{k=1}^{|S_i|}p_{i,k}\cdot X_{i,k}}{\sum\limits_{l=1}^{|S_j|}p_{j,l}\cdot X_{j,l}} \geq c,\\
\text{minimise } \sum\limits_{d\in D} \mathrm{Xor}(X_{i,j},T(d)).
\end{cases}
\]
The first line is the fairness requirement. For any two sensitive groups, the difference in their passing rates should be bounded by a multiplicative factor. The second line captures the semantic difference requirement:
minimise the number of data points that have a different classification outcome after the repair.
This set of formulas always has a solution, as the fairness requirement constraint has a trivial solution of $X_{i,j} = 1$ for all $i,j$. Since the solution space is non-empty, an optimal solution exists. We omit the definition of optimal solution for a random forest as it can be derived by generalising the above.




\section{FairRepair for Decision Tress}
\label{main algo}

The goal of our approach is to output a decision tree or a random forest, such that it satisfies the fairness requirement, and when compared to the optimal solution, the semantic difference with the output tree or forest is bounded by a multiplicative coefficient $\alpha$. Our approach will be able to output a solution for any $\alpha>1$. 



\subsection{Top-level Algorithm}

Fairness criteria are inequalities over path probabilities.  
A probability triple is $(\Omega, \mathcal{F}, P)$, where $\Omega$ is the sample space, $\mathcal{F}$ is the event space and $P$ is the probability function. As defined in Section~\ref{subsec: defs},  $\Omega$ is the Cartesian product of discrete spaces (for discrete attributes) and 1-dimensional continuous spaces (defined on $\mathbb{R}$, for continuous attributes). 

Algorithm \ref{alg:main} outlines the repair algorithm. We first collect all decision tree {\em path hypercubes}, and group them based on their sensitive attributes. Each {\em path hypercube} represents a region in the input space that is specified by its path constraints. The next step is to calculate the path probabilities, which are defined to be the proportion of inputs that enter the path hypercube. We assume that each path hypercube has a 1/0 outcome, depending on whether the resource is allocated. We specify 1 as the desired return value when we compute passing rates, which are part of the fairness inequalities as described in Section~\ref{subsec: defs}. In Algorithm~\ref{alg:main}, we use linear optimisation to calculate how much we want to change the {\em passing rates} to meet the fairness requirement. The constraints are sent to a partial MaxSMT solver. For some path hypercubes, we flip the outcomes to meet the fairness requirement. For some path hypercubes, we refine them (by inserting additional conditions) to meet the desired passing rates. The SMT procedure terminates when a solution is found. Finally, we modify the decision tree based on the solution from the SMT solver.

\begin{algorithm}
\caption{Top-level algorithm in {\sf FairRepair} for Decision Trees\label{alg:main}}
\begin{algorithmic}
\renewcommand{\algorithmicrequire}{\textbf{Input:}}
\renewcommand{\algorithmicensure}{\textbf{Output:}}
\REQUIRE $D,D_0 = \text{datasets}$, $c =\text{fairness threshold}$.
\ENSURE Solution for Repaired Tree.

    \STATE \textbf{train} $T_0$ on $D_0$  \COMMENT{$T_0$ is a decision tree trained on $D_0$.}
    \STATE T $\leftarrow$ \textbf{Collectpath hypercubes}($T_0$) \COMMENT{Decisions on sensitive attributes.}
    \STATE $[H]$ $\leftarrow$ \textbf{Tree2HCubes}($T$) \COMMENT{Path constraints.}
    \STATE $[P]$ $\leftarrow$ \textbf{PathRateCalculator}($D$, $H$)
    \COMMENT{Path probabilities.}
    \STATE $[R]$ $\leftarrow$ $\textbf{LinearOpt}([P])$ \COMMENT{Desired passing rates.}

    \WHILE{isRefinable([H]) == $\mathtt{TRUE}$} 
        \IF{\textbf{MaxSMT}([H], [P], [R]) == \texttt{UNSAT}} 
        \STATE \textbf{Refine}([H]) \COMMENT{Refine path hypercubes.}
        \ELSE{\RETURN \textbf{MaxSMT}([H], [P], [R])} 
        \ENDIF
    \ENDWHILE
    \WHILE{$\mathtt{TRUE}$}
        \STATE\textbf{RelaxSemDiffConstraint()}
        \IF{\textbf{MaxSMT}([H], [P], [R]) == \texttt{SAT}}{\RETURN\textbf{MaxSMT}([H], [P], [R])}
        \ENDIF
    \ENDWHILE

\end{algorithmic}
\end{algorithm}

\subsection{Collecting Decisions on Sensitive Attributes}
\label{subsec: sens collect}


Given a dataset $D_0$, we first train a decision tree $T$ on $D_0$ with an external tool. The algorithm starts by collecting the path hypercubes and grouping them based on the sensitive attributes in their path constraints. Let $\{A_1, \dots, A_m\}$ be the set of sensitive attributes. There are in total $M:= |A_1| \times |A_2| \times \cdots \times |A_m|$ sensitive groups, where each $|A_i|$ is the number of partitions of the valuations of $A_i$. For each sensitive group, create an empty set $S_i$.  We use $t_i$, $1\leq i\leq M$, to denote the sensitive attribute values of a sensitive group. This does not necessarily require the sensitive attributes to have discrete domains, we only need a discrete partitioning of the set of valuations of any sensitive attribute. 

A path hypercube of the decision tree $T$ corresponds to a region in the input space. Similarly, each sensitive group is also a region in the input space. For each path hypercube, its intersection with the sensitive group $t_i$ will be recorded to the corresponding set $S_i$. Thus, each set $S_i$, $1 \leq i \leq M$, contains all path hypercubes that an input with sensitive attributes $(t_i)$ could possibly traverse. Modifying the outcome of a path hypercube in set $S_i$ has no effect on the path hypercubes in other sets, as the regions in the input space that these sets cover are mutually disjoint. 


\subsection{Path Probability Calculation}

To conduct the repair, we first evaluate the tree with respect to the fairness requirements defined in Section~\ref{subsec: defs}. 
In particular,  $\forall 1 \leq i, j \leq M$, $r_i/r_j > c$, i.e., the passing rates of any two sensitive groups should be similar (up to threshold $c$). Let $D$ be a dataset. Note that we do not require $D=D_0$. We assume that the data set represents the input distribution, and define path probability to be the number of data points that satisfy the path. Each $r_i$ is then obtained by adding up the path probabilities of those path hypercubes returning the required outcome in set $S_i$. 
Also, we define the quantity $p_i = \sum_{\pi \in S_i} prob(\pi)$
where $prob(\pi)$ is the path probability of path $\pi$.
Thus, for each path hypercube $\pi$ in $S_i$, we count the number of data points in $\pi$. This number divided by the size of the data set is the path probability $prob(\pi)$. Using these path probabilities,  we can calculate the passing rates $r_i$ and proportions $p_i$. In our example, $r_\text{male}$ is the proportion of male applicants getting a loan. This value is the sum of the path probabilities of all path hypercubes in $S_\mathtt{male}$ that returns $\mathtt{TRUE}$. If we sum up all path probabilities in $S_\mathtt{male}$, we obtain the probability of $p_\text{male} = \mathbb{P}(\text{sex = male})$.

\subsection{Calculating theoretical optimal passing rates}
\label{lo computation}
We now have the passing rates $r_i$ for each set $S_i$, $1\leq i \leq M$, i.e., each pair of sensitive attributes. Recall the fairness requirement: $\forall 1 \leq i, j \leq M$, $\dfrac{r_i}{r_j} > c$, where $r_i$ and $r_j$ are passing rates of subtrees. To meet this requirement, we need to modify the decision tree such that for each $1\leq i \leq M$, the passing rate of $T_i$ changes from $r_i$ to $x_i$.

In Section~\ref{subsec: defs} we directly solved for the return values of each data points that achieves an optimal repair. By contrast, here we only compute the optimal passing rate change constrained by the fairness requirements.
Additionally, we aim to find a solution with the minimal semantic difference from the original, unfair, decision tree.  
To find the $x_i$s, we solve the following linear optimization problem:

\[
\begin{cases} 
\forall \,1 \leq i,j\leq M, \dfrac{x_i }{x_j} \geq c , \\
\forall \,1 \leq i\leq M, 0 \leq x_i \leq 1, \\
\text{minimise } \sum\limits_{i=1}^{M}p_i\cdot|x_i-r_i|.
\end{cases}
\]

 \noindent The first line is the fairness requirement. Since passing rates are probabilities, the second line requires them to be bounded by 0 and 1. The third line accounts for the semantic difference. Each $p_i\cdot|x_i-r_i|$ is a lower bound for the proportion of data points being affected in set $S_i$. It is possible that the passing rates $x_i$ and $r_i$ corresponds to different groups of data points, and the actual number of data points being affected is not $|x_i -r_i|$, but $x_i + r_i$. Thus, $p_i\cdot|x_i - r_i|$ is only a lower bound on the proportion of data points affected in $S_i$.
 
 In our example, suppose for male applicants, initially only those who have high education receive a loan. When this is changed to only those who have medium education receive a loan, the difference in passing rate is only $|\mathbb{P}(\text{male, high}) - \mathbb{P}(\text{male, med})|$, but the proportion of data points being affected is $\mathbb{P}(\text{male, high})+ \mathbb{P}(\text{male, med})$. 

The optimisation problem always has a solution, giving us the ``desired passing rates" $x_1, x_2, \ldots, x_M$ which are used to modify the path hypercubes in the decision tree, as discussed in the following.



\subsection{Calculating patches with MaxSMT}
\label{subsec:maxsmt}

Our goal is to output a decision tree, such that when compared to the optimal solution, the semantic difference of our tree is bounded by a multiplicative coefficient $\alpha$. This can be achieved by using a partial MaxSMT solver. Let $P_{i,j}$'s and $p_{i,j}$'s denote the path hypercubes and path probabilities of set $S_i$. Let $I_{i,j}$ be a pseudo-Boolean value that indicates if the return value of $P_{i,j}$ is $\mathtt{TRUE}$ or $\mathtt{FALSE}$. For each $I_{i,j}$, let $X_{i,j}$ be a pseudo-Boolean SMT variable. The solver will solve for these $X_{i,j}$'s. The idea is that, to meet minimum semantic difference, most $X_{i,j}$ should be the same as $I_{i,j}$, but the solver could produce a few $X_{i,j}$ which are different from $I_{i,j}$; these are the path hypercubes in the decision tree where the classification outcome of the path is flipped. 
We now discuss the partial MaxSMT problem formulation which produces the $X_{i,j}$ values. Our MaxSMT problem has hard constraints, which must be met, and soft constraints the maximal number of which should be met.

\paragraph{Hard Constraints}
The first set of hard constraints are the fairness requirements. For every pair of subtrees, the ratio of their passing rates must be bounded. Recall that $S_i$ is the set of path hypercubes of sensitive group $i$.
\[
\forall\, 1\leq i, j\leq M, k, l, \frac{\sum\limits_{k=1}^{|S_i|}p_{i,k}\cdot X_{i,k}}{\sum\limits_{l=1}^{|S_j|}p_{j,l}\cdot X_{j,l}} \geq c.
\]
We account for semantic difference in the hard constraints.
\[
\sum\limits_{i = 1}^{M}\sum\limits_{j = 1}^{|S_i|} \,p_{i,j}\cdot \mathrm{Xor}(X_{i,j},I_{i,j})\leq \alpha\sum\limits_{i = 1}^{M}p_i\cdot|x_i-r_i|.
\]
The left hand side is the semantic difference between our modified tree and the original tree. The right hand side is the minimal possible semantic difference, multiplied by $\alpha$. 

\paragraph{Soft Constraints} The soft constraints put $X_{i,j}$ to be same as $I_{i,j}$. The solver should satisfy the maximal number of these soft constraints, which means that we should flip the return value of as few path hypercubes as possible.
\[
\forall\, i, j,\, X_{i,j} = I_{i,j}.
\]
The solver returns $\mathtt{UNSAT}$ whenever our semantic difference is too large. 
Note that the above set of formulas are not guaranteed to have a solution, because the path probabilities might be too large. If the solver returns $\mathtt{UNSAT}$ for the partial maxSMT problem, then we cannot make the decision tree fair by flipping path hypercube return values. In that case, our algorithm moves on to refine the path hypercubes, a process we describe next.

\subsection{Refining Path Hypercubes}
\label{subsec: refine}

If the fairness criteria cannot be satisfied by flipping the outcome of the path hypercubes (i.e., the hard constraints are unsatisfiable), we proceed to refine the path hypercubes. In particular, we split one path hypercube into two path hypercubes based on a single attribute. There is no restriction whether this attribute should be discrete or continuous. When the solver outputs $\mathtt{UNSAT}$, we refine a single path hypercube and re-run the solver. Since the total number of path hypercubes is bounded by the size of the dataset, this procedure always terminates after a finite number of steps.

Refinement requires choosing a path hypercube and a hypercube constraint on an attribute. FairRepair delegates this process to the SMT solver. The changes in passing rates for all possible constraints to refine the hypercube are pre-calculated. Then, we iteratively refine path hypercubes by encoding each refinement as an optimisation problem: at each iteration we select a refinement that maximally improves fairness. After a refinement, we check if the new decision tree meets the fairness criteria. If not, we refine again. To decide on a constraint to refine, we consider the following two cases: attributes with discrete domains and attributes with ordered domains.

\paragraph*{Discrete attributes} For attributes with discrete domains, such as boolean attributes, we consider splitting path hypercubes by choosing arbitrary subsets of values. Specifically, let $P_{i,j}$ be a path hypercube to refine, and let $A$ be an attribute, and $\mathit{Dom}(A) \coloneqq \{a_1, a_2, ..., a_n\}$ be the domain of $A$. Our goal is to divide the domain $\mathit{Dom}(A)$ into two disjoint sets $\mathcal{A}_T$ and $\mathcal{A}_F$ that would corresponds to the refinements of the path hypercube $P_{i,j}$ into $P_{i,j}^T$ and $P_{i,j}^F$, so that the resulting decision tree returns $\mathit{True}$ for all inputs with the value of $A$ in $\mathcal{A}_T$, and returns $\mathit{False}$ for all inputs with $A$ values in $\mathcal{A}_F$.
To realize this, we encode path probabilities for all subsets of attribute values. For each value $a_l$ in $\mathit{Dom}(A)$ we introduce a boolean variable $Y_{i,j,l}$ that indicates if this value should be in the set $\mathcal{A}_T$ (i.e., if the decision for $a_i$ should be $\mathit{True}$). We also assume that $p(P_{i,j}, a_l)$ indicates the passing rate for inputs in $P_{i,j}$ that have the attribute $a_l$. Then, we formulate the refinement problem as the following SMT query:
\[
\forall\, 1\leq i, j\leq M, k, l, \frac{\sum\limits_{k=1}^{|S_i|}\sum\limits_{l=1}^{|\mathit{Dom}(A)|}p(P_{i,k}, a_l)\cdot Y_{i,k,l}}{\sum\limits_{k=1}^{|S_j|}\sum\limits_{l=1}^{|\mathit{Dom}(A)|}p(P_{j,k}, a_l)\cdot Y_{j,k,l}} \geq c.
\]
\noindent The semantic difference hard constraint in Section~\ref{subsec:maxsmt} remains unchanged. We also add soft constraints indicating that we should refine as few hypercubes as possible. That is, for each path hypercube $P_{i,j}$, either $\bigwedge_{l} Y_{i,j,l}$ or $\bigwedge_{l} \neg Y_{i,j,l}$ depending on the original decision.
Given a solution to the above constraints, we can reconstruct the sets $\mathcal{A}_T$ and $\mathcal{A}_F$ from the values of $Y_{i,j,l}$, and refine the corresponding path hypercubes.

\paragraph*{Ordered domains} Ordered domains, such as integers and real numbers, require a refinement approach based on an inequality predicate. For each path hypercube $P_{i,j}$, we generate $E_{i,j}$, the set of all possible expressions of the form $\mathit{variable} > \mathit{const}$ or $\mathit{variable} \leq \mathit{const}$, where each $\mathit{variable}$ is a data attribute, and $\mathit{const}$ is a value of this attribute corresponding to a data point that lies in the path hypercube. For each expression $e_{i,j,l} \in E_{i,j}$, we compute $p(P_{i,j}, e_{i,j,l})$ the passing rate of inputs in $P_{i,j}$ that satisfy the predicate $e_{i,j,l}$. Then, we solve the same formula as for the unordered case:
\[
\forall\, 1\leq i, j\leq M, k, l, \frac{\sum\limits_{k=1}^{|S_i|}\sum\limits_{l=1}^{|E_{i,j}|}p(P_{i,k}, e_{i,j,l})\cdot Y_{i,k,l}}{\sum\limits_{k=1}^{|S_j|}\sum\limits_{l=1}^{|E_{i,j}|}p(P_{j,k}, e_{i,j,l})\cdot Y_{j,k,l}} \geq c.
\]

\noindent The semantic difference hard constraint in Section~\ref{subsec:maxsmt} remains unchanged. We also ensure that at most one expression is used in refining each hypercube by including a cardinality constraints $\mathit{atMostOne}(e_{i,j,1}, ..., e_{i,j,n})$ for each path hypercube.

\subsection{Relaxing the Semantic Distance Constraint}
\label{subsec:relax}

When all the path hypercubes are fully refined, i.e., all data points residing in a path hypercube have exactly the same attribute values, refinement stops. 
At this point we adopt an alternative method. We relax the semantic distance constraint gradually, until a repair that meets the fairness requirement is found. The SMT formula we use here remains the same as in Section~\ref{subsec:maxsmt}, except for the soft constraints below.

\paragraph{Hard Constraints}
The first set of hard constraint are the fairness requirements.
\[
\forall\, 1\leq i, j\leq M, k, l, \frac{\sum\limits_{k=1}^{|S_i|}p_{i,k}\cdot X_{i,k}}{\sum\limits_{l=1}^{|S_j|}p_{j,l}\cdot X_{j,l}} \geq c.
\]
The second hard constraint is the semantic difference requirement. Note that the soft constraints in Section~\ref{subsec:maxsmt} have now been merged into this second hard constraint, because now the path hypercubes are data points themselves, and minimising the number of change to path hypercubes is the same as minimising the number of change to data points.
\[
\sum\limits_{i = 1}^{M}\sum\limits_{j = 1}^{|S_i|} \,p_{i,j}\cdot \mathrm{Xor}(X_{i,j},I_{i,j})\leq \alpha\cdot \mathrm{SemDiff}.
\]
Initially, $\mathrm{SemDiff}$, the semantic difference, is set to $\sum\limits_{i = 1}^{M}p_i\cdot|x_i-r_i|$, the theoretical minimal semantic difference computed in Section \ref{lo computation}. If the solver returns $\mathtt{UNSAT}$, we relax the semantic difference constraint by updating $\mathrm{SemDiff}$ to be $\mathrm{SemDiff}^{'}= \alpha\cdot\mathrm{SemDiff}$. That is,
\[
\sum\limits_{i = 1}^{M}\sum\limits_{j = 1}^{|S_i|} \,p_{i,j}\cdot \mathrm{Xor}(X_{i,j},I_{i,j})
\leq \alpha \cdot \mathrm{SemDiff}^{'} =\alpha \cdot\alpha\cdot \mathrm{SemDiff} .
\]
Next, we re-run the solver. We iteratively update $\mathrm{SemDiff}$ until the solver finds a solution of $X_{i,j}$'s.
To see that our algorithm always terminates, note that there exists a trivial solution for the fairness constraints, i.e., all $X_{i,j}$'s are $\mathtt{TRUE}$. 
The left hand side of the above inequality represents the actual semantic change in fraction, so it is bounded by 1. 
Meanwhile, since $\alpha > 1$, as we repeatedly relax $\mathrm{SemDiff}$, the right hand side of the inequality is unbounded and will exceed 1 after finitely many iterations.
Thus the solver is guaranteed to find a solution after a finite number of steps. This concludes our algorithm.

\section{FairRepair for Random Forest}
\label{forest}
The algorithm in Section \ref{main algo} can be extended to repair random forests. A random forest is an ensemble of decision trees, For a given input, its outcome is the majority vote (for classification) or average (for regression) of the outcomes from the decision trees. In our case, the outcome is binary, so we use the majority vote definition. Let $F:=\{T_1,T_2,\dots,T_n\}$ be a random forest with $n$ decision trees. For input $x$,
\[
F(x) = \maj\limits_{1\leq i\leq n} T_i(x),
\] 
where
\[
\maj\limits_{1\leq i\leq n} X_i = \begin{cases}
1, \text{if }\sum X_i > \dfrac{n}{2}\\
0, \text{otherwise}.
\end{cases}
\]
Given a training dataset $D_0$, we first train a random forest $F$ with an external tool. We extract the decision trees $T_1,\dots,T_n$ from $F$, and collect the path hypercubes as in Section \ref{subsec: sens collect}. Let $\{A_1,\dots,A_m\}$ and $M$ be defined as earlier. For each decision tree $T_i$, we collect its path hypercubes and group them based on the sensitive attributes in their path constraints. Let $t_j$ denote the sensitive attribute values of a sensitive group, and let $S_{i,j}$ denote the set of all path hypercubes that an input with sensitive attributes $t_j$ could possibly traverse in tree $T_i$. Each path hypercube is given by a tuple $(i,j,\mathrm{hid})$, indicating the tree $T_i$ it belongs to, the sensitive group $t_j$ it resides in, and unique index ``hid'' to identify the path hypercube from $S_{i,j}$.
We next calculate the passing rates of each sensitive group in the random forest, by recording the return value of every input data point. Since the optimal passing rates are purely determined by the classification results, we have the same linear optimisation problem for a random forest as for a decision tree. Solving the linear optimisation problems, we obtain the theoretical optimal passing rates $x_i$'s.


\subsection{Computing Patches for the Initial Decision Trees}
We now extend our SMT formulas to random forests. Let $D$ denote a dataset. Note that we do not require $D=D_0$.
%
%
The outcome for an input $d\in D$ is the majority vote of the outcomes of path hypercubes where point $d$ resides, across all decision trees. The point $d$ resides in the intersection of these path hypercubes. If we compute the intersections of all path hypercubes over all decision trees (for a random forest with 30 decision trees each with 6,000 paths in our case) the total number of intersections quickly explodes. 
%
%
We overcome this by only considering the intersections containing at least one datapoint from the dataset. This follows from our assumption that the dataset $D$ represents the input distribution.
Instead of finding intersections that contains datapoints, we start from the datapoints to find the intersections. 
%
%
For each datapoint $d\in D$, we identify the path hypercube where $d$ resides in each decision tree. For decision tree $T_i$, let the index tuple of the path hypercube that contains $d$ be $(i,j_d,\mathrm{hid}_{i,d})$, the intersection $I_d$ is
\[
I_d:=\bigcap_{1\leq i\leq n} \mathrm{P}_{i,j_d,hid_{i,d}},
\]
where $\mathrm{P}_{i,j_d,hid_{i,d}}$ is the path hypercube in $S_{i,j}$ with $j_d$ being the sensitive group $d$ belongs to, and $hid_{i,d}$ being the unique index of the path hypercube. For all $d\in D$, we compute $I_d$ and record $(i,j_d,hid_{i,d})$ for each $1\leq i\leq n$. Note that, it is possible that $I_{d_1} = I_{d_2}$ for $d_1\neq d_2$. Recall that $t_j$'s denote the sensitive groups, and $s_j$ denotes the set of datapoints in the sensitive group $t_j$. We shall classify datapoints in each $s_j$ into equivalence classes based on the intersection they reside in. We denote $d_1\sim d_2$ if $I_{d_1} = I_{d_2}$, and we use $d_1$ to denote the equivalence class without loss of generality. We abuse the notation to let $s_j$ also denote the equivalence classes. To indicate the number of datapoints within an intersection $I_{d_1}$, we assign a \textit{frequency} field $f_{d_1}$ to each intersection, denoting the number of datapoints in the equivalence class.

Next, we construct the SMT formulas. For each path hypercube $\mathrm{P}_{i,j,hid}$ in the random forest $F$, assign an SMT Boolean variable $X_{i,j,hid}$ as the return value of the path hypercube. We use $d$'s to denote the equivalence classes in $s_j$, and use $|s_j|$ to denote the number of datapoints in $s_j$. We then formulate the fairness requirement as follows:

\[
\forall\, 1\leq  k,l\leq M, \dfrac{\sum\limits_{d\in s_k}\maj\limits_{1\leq i\leq n}(X_{i,k,hid_{i,d}})\cdot f_{d}/|s_j|}{\sum\limits_{d\in s_l}\maj\limits_{1\leq i\leq n}(X_{i,l,hid_{i,d}})\cdot f_{d}/|s_k|} \geq c.
\]
The semantic hard constraint is:
\[
\sum\limits_{s_k\in D}\sum\limits_{d\in s_k} f_d\cdot\mathrm{Xor}(\maj\limits_{1\leq i\leq n}(X_{i,k,hid_{i,d}}),F(d))\leq \alpha\sum\limits_{i = 1}^{M}p_i\cdot|x_i-r_i|,
\]
where $s_k$'s partition the dataset $D$. Note that for all $d$ within an equivalence class, they belong to the same path hypercube in each decision tree. Hence they agree on $F(d)$.

\subsection{Computing Patches for the Intersections}

If the result of the above SMT constraints is $\mathtt{UNSAT}$, we treat each intersection path hypercube separately. Let $S_j:= \{I_d:d\in s_j\}$ for $1\leq j \leq M$. Assign a pseudo-Boolean SMT variable for each $I_d$ for $d\in D$. We abuse notation by letting $I_d$ also represent the initial outcome of the intersection. Let $f_d$ still denote the path probability of $I_d$. Now it is almost the same as repairing a decision tree, and the following captures the fairness requirement.
\[
\forall\, 1\leq i, j\leq M,  \frac{\sum\limits_{d\in s_i}f_{d}\cdot X_d}{\sum\limits_{d^{'}\in s_j}f_{d^{'}}\cdot X_{d^{'}}} \geq c.
\]
The semantic hard constraint is:
\[
\sum\limits_{j=1}^{M}\sum\limits_{I_d\in S_j} f_d\cdot\mathrm{Xor} (I_d,X_d)\leq \alpha\sum\limits_{i = 1}^{M}p_i\cdot|x_i-r_i|.
\]

If the above SMT constraints are $\mathtt{SAT}$, we update the changes in the intersections in the initial decision tree path hypercubes. If $I_d=X_d$, we do not change anything. If $I_d\neq X_d$, the majority vote in the intersection is changed and we compute the minimum number of path hypercubes whose outcomes need to be flipped to alter the majority vote. Recall the forest $F$ has $n$ trees. Let $I_{i,j_d,hid_{i,d}}$ denote the outcomes of the path hypercubes $\mathrm{P}_{i,j_d,hid_{i,d}}$ for $1\leq i \leq M$, $d\in D$. Then $\maj\limits_{1\leq i\leq M}(I_{i,j_d,hid_{i,d}}) = I_d$. To flip $I_d$, randomly select minimal number of decision trees $T_i$ with $I_{i,j_d,hid_{i,d}}\neq X_d$ and add the path hypercube $I_d$ with prediction $X_d$ to $T_i$, such that $\maj\limits_{1\leq i\leq M}(I_{i,j_d,hid_{i,d}})$ is flipped. If $I_d = \mathtt{TRUE}$, minimally $\sum\limits_{1\leq i\leq M}(I_{i,j_d,hid_{i,d}})-\left \lfloor{\dfrac{n}{2}}\right \rfloor $ decision trees with $I_{i,j_d,hid_{i,d}} = \mathtt{TRUE}$ will have added a new path hypercube $I_d$ (the intersection path hypercube) with a $\mathtt{False}$ prediction. If $I_d = \mathtt{FALSE}$, minimally $\left \lceil{\dfrac{n}{2}}\right \rceil -\sum\limits_{1\leq i\leq M}(I_{i,j_d,hid_{i,d}})$ decision trees with $I_{i,j_d,hid_{i,d}} = \mathtt{FALSE}$ will have added a new path hypercube $I_d$ with a $\mathtt{TRUE}$ prediction. 
%

If the above SMT constraints are still $\mathtt{UNSAT}$, we proceed to refine the intersections in the same manner as we refined the decision trees in Section~\ref{subsec: refine}.

\section{Relative Soundness and Completeness Proof}
\label{sec:complete}

In this section, we present a proof of relative soundness and completeness of our repair technique. This completeness property distinguishes our work from existing decision tree repair tool DIGITS~\cite{CAV17}. The only assumption we make is the soundness and completeness of the SMT solver: if the solver is sound and complete, then our algorithm is also sound and complete. The proof for decision trees and random forests are similar. Here we omit the proof for random forests.

The soundness of our approach partly follows from the soundness of the SMT solver. If the solver returns a solution, it must satisfy the hard constraints of the fairness requirement and the (possibly relaxed) semantic change constraint. We explained how our algorithm always provide a solution for a repaired model in \ref{subsec:relax}. We now prove that whenever our algorithm produces a repaired model, its fairness and semantic difference is always guaranteed. In practice, we also compute the ratios of the passing rates to check if they have passed the fairness threshold. We refer our readers to \ref{evaluation} for the experimental results. 


\begin{lemma}
 For a given decision tree $T$ obtained via training on a given dataset $D$,
 for any $\alpha>1$, if the group fairness requirement is satisfiable and there exists an optimal repair $R_{op}$, our algorithm (Section~\ref{main algo}) is able to find a repair $R_{subop}$ satisfying the fairness requirement, such that $SD(R_{subop},T) \leq \alpha SD(R_{op},T)$. Here $R_{subop}$ refers to a solution (assignment) of the variables $X_{i,j}$'s, and $SD$ is the semantic distance function defined in Section \ref{subsec: defs}. For any assignment $R$ of $X_{i,j}$'s, 
\[
SD(R,T) = \sum\limits_{i = 1}^{M}\sum\limits_{j = 1}^{|S_j|} \,p_{i,j}\cdot \mathrm{Xor}( X_{i,j},I_{i,j}),
\]
where the symbols are defined in Section~\ref{main algo}.
\end{lemma}

\begin{proof}
In Section \ref{subsec:relax}, we have shown that our algorithm always terminates with an output repair whose fairness is hard-constrained by the SMT solver. It suffices to show that repair satisfies the semantic requirement even after relaxation of $\alpha$.

Our algorithm first checks if the group fairness requirement can be achieved by only flipping the existing path hypercubes as described in Section~\ref{subsec:maxsmt}. If the solver returns $\mathtt{SAT}$, our algorithm outputs the assignments of $X_{i,j}$'s and terminates. 
If the fairness requirement is $\mathtt{UNSAT}$, our algorithm proceeds to refine the path hypercubes, checks if the fairness requirement is $\mathtt{SAT}$ after the refinement, until all the path hypercubes are fully refined (Section~\ref{subsec: refine}). Once all the hypercubes are fully refined, the path hypercubes are equivalent to the input data points, meaning each path hypercube denotes a single data point, with its duplicates recorded as path probability.
Our algorithm then relaxes the semantic distance constraints until the solver returns $\mathtt{SAT}$ (Section~\ref{subsec:relax}), outputs the solution, and terminates.

A solution could be produced in any of the above three stages. We now show that whenever our algorithm outputs a solution $R_{subop}$, $R_{subop}$ satisfies $SD(R_{subop},T) \leq \alpha SD(R_{op},T)$. To distinguish $R_{subop}$ from $R_{op}$, let $R_{op} =\{x_{i,j}:1\leq i\leq M,1\leq j \leq |S_i|\}$, and $R_{subop} = \{y_{i,j}:1\leq i\leq M,1\leq j \leq |S_i|\}$ such that both are assignments of the variables $X_{i,j}$'s. 

\paragraph{Before Refinements} If the fairness requirement can be achieved by only changing the outcomes of some path hypercubes, without splitting any of them, as a hard constraint (Section~\ref{subsec:maxsmt}), we have 
\[
\sum\limits_{i = 1}^{M}\sum\limits_{j = 1}^{|S_j|} \,p_{i,j}\cdot \mathrm{Xor}( y_{i,j},I_{i,j})\leq \alpha\sum\limits_{i = 1}^{M}p_i\cdot|x_i-r_i|.
\]
The left hand side is $SD(R_{subop},T)$, and the right hand side is the product of $\alpha$ and the theoretical minimal semantic difference, which is a lower bound for the actual semantic difference. That is,
\[
\sum\limits_{i = 1}^{M}\sum\limits_{j = 1}^{|S_j|} \,p_{i,j}\cdot \mathrm{Xor}( x_{i,j},I_{i,j}) \geq \sum\limits_{i = 1}^{M}p_i\cdot|x_i-r_i|.
\]
Therefore,
\begin{align*}
SD(R_{subop},& T) 
= \sum\limits_{i = 1}^{M}\sum\limits_{j = 1}^{|S_j|} \,p_{i,j}\cdot \mathrm{Xor}( y_{i,j},I_{i,j})\\
\leq & \,\alpha \sum\limits_{i = 1}^{M}\sum\limits_{j = 1}^{|S_j|} \,p_{i,j}\cdot \mathrm{Xor}x_{i,j},I_{i,j}) =\alpha SD(R_{op},T).
\end{align*}

\paragraph{During Refinements} If the fairness requirement is not achieved by only changing the outcomes of the hypercubes, we refine the path hypercubes and call the solver again. Note that $X_{i,j}$'s are now different, since $|S_i|$ increases as the number of path hypercubes increases due to refinement. We use $I^{'}_{i,j}$'s to denote the outcomes of the new path hypercubes, to distinguish the path hypercubes after refinement from those before the refinements. Nonetheless, the outcomes of the path hypercubes before and after a refinement do not conflict with each other, as each path hypercube subset inherits the outcome from the original one. Since in the SMT formula construction, the same semantic difference hard constraint is inherited from the previous step, we also have
\begin{align*}
SD(R_{subop},& T) 
= \sum\limits_{i = 1}^{M}\sum\limits_{j = 1}^{|S_j|} \,p_{i,j}\cdot \mathrm{Xor}( y_{i,j},I^{'}_{i,j})\\
\leq & \,\alpha \sum\limits_{i = 1}^{M}\sum\limits_{j = 1}^{|S_j|} \,p_{i,j}\cdot \mathrm{Xor}( x_{i,j},I^{'}_{i,j}) =\alpha SD(R_{op},T).
\end{align*}

\paragraph{After Refinements} 
When all the path hypercubes are fully refined, $X_{i,j}$'s effectively represent the classifications of each individual point. We shall still use $I_{i,j}^{'}$'s to represent the original return values for each data point. After the path hypercubes have been fully refined, if the SMT solver still returns \texttt{UNSAT}, it means that the theoretical minimal semantic difference cannot be achieved. Here we use the completeness of the solver. Thus, we know that

\[
\sum\limits_{i = 1}^{M}\sum\limits_{j = 1}^{|S_j|} \,p_{i,j}\cdot \mathrm{Xor}( X_{i,j},I^{'}_{i,j}) > \alpha\cdot \mathrm{SemDiff},
\]
for all assignments for $X_{i,j}$'s that satisfy the fairness requirement. In particular,
\[
SD(R_{op},T) = \sum\limits_{i = 1}^{M}\sum\limits_{j = 1}^{|S_j|} \,p_{i,j}\cdot \mathrm{Xor}(x_{i,j},I^{'}_{i,j}) > \alpha\cdot \mathrm{SemDiff}.
\]
We use $p_{i,j}$'s to denote the frequency of repeated data points. As explained in Section~\ref{subsec:relax}, the relaxing procedure always terminates. Let $\mathrm{SemDiff_0}$ be the $\mathrm{SemDiff}$ when the solver return $\mathtt{UNSAT}$ at the last time. Then we have the following:
\[
SD(R_{op},T) = \sum\limits_{i = 1}^{M}\sum\limits_{j = 1}^{|S_j|} \,p_{i,j}\cdot \mathrm{Xor}(x_{i,j},I^{'}_{i,j}) > \alpha\cdot \mathrm{SemDiff_0}, \text{ and }
\]
\[
SD(R_{subop},T)= \sum\limits_{i = 1}^{M}\sum\limits_{j = 1}^{|S_j|} \,p_{i,j}\cdot \mathrm{Xor}(y_{i,j},I^{'}_{i,j}) \leq \alpha\cdot\alpha\cdot \mathrm{SemDiff_0}.
\]
The above inequalities imply 
\[
SD(R_{subop},T)\leq \alpha \cdot SD(R_{op},T).
\]

\end{proof}

\section{Implementation}

We developed FairRepair in 2,500 lines of Python code. FairRepair repairs decision trees trained using scikit-learn~\cite{scikit-learn} \emph{DecisionTreeClassifier}. In general, our algorithm is not restricted to any specific decision tree or random forest training tool.
FairRepair changes the structure of the decision tree, but it does not change the output range of classifications or assume new input features. 
We now briefly describe the implementation.


\paragraph{Tree to Hypercubes Conversion}
Decision trees and random forests trained by scikit-learn tool do not support categorical attributes directly. By default, it assumes all attributes are continuous. This structure cannot be used for MaxSMT solving directly, and we need to convert it into path hypercubes. We use \textit{OneHot Encoding} to generate $n=|A|$ attribute names $A_1, \dots, A_n$ for a categorical attribute $A$. The OneHot threshold is set to be 0.5 and the data points' valuation on these attributes are 0 or 1. Note that,
\begin{align*}
\forall \text{ discrete }& A, \text{if } A = \{A_1,\dots,A_n\}, \text{then }\forall d \in D. \exists \, i,1 \leq i \leq n. \\
& d[A_i] = 1 \wedge (\forall A' \in \{A_1,\dots,A_n\}\backslash {A_i}.\, d[A'] = 0 ),
\end{align*}
as each data point has a unique valuation for the attribute $A$.

\paragraph{Hypercube Mapping on Sensitive Attributes}
To obtain the intersections of path hypercubes with sensitive groups of multiple sensitive attributes, we implemented the cross product of path hypercube sets. For two sensitive attributes $A = \{a_1,\dots,a_m\}$, $B=\{b_1,\dots,b_n\}$, we first collect path hypercube set $\mathcal{S}_A = \{S_{a_1}, \dots,S_{a_m}\}$ and $\mathcal{S}_B = \{S_{b_1},\dots\,S_{b_n}\}$ as described in Section~\ref{subsec: sens collect}. Then we compute the cross product $\mathcal{S}_A\times\mathcal{S}_B := \{S_{a_i,b_j}:1\leq i\leq m,1\leq j \leq n\}$, where each $S_{a_i,b_j} := \{h\cap h^{'} : h \in S_{a_i}, h^{'} \in S_{b_j}\}$. This cross product computation can be easily generalised to multiple attributes.

\paragraph{Linear Optimisation} 
With the path hypercubes we can calculate the path probabilities and passing rates for each sensitive group. We use the Python library \textit{pulp} to calculate the theoretical minimal changes in passing rates using linear optimisation. Fairness requirements are implemented as inequalities, where each variable denotes the passing rate of a sensitive group. We added additional variables to account for absolute values, which is not supported by \textit{pulp} directly, and omitted the sensitive groups with too few data points (less than ten).  

\paragraph{SMT Formulas Construction}
We use the Z3 SMT solver. All fairness requirements are constructed as linear arithmetic formulas. To improve performance, rational fairness threshold $c$ is represented as the ratio of two integers $(a,b)$. 
For each path hypercube, we use $(n+1)$ variables $X, X_1,\dots,X_n$, where $n$ is the number of data points residing in the path hypercube. Variable $X$ (\textit{path variable}) denotes the return value of the path hypercube, and $X_i$'s (\textit{point variables}) denote the data points in the path hypercube. Each $X_i$ is encoded to have the same value as the $X$. Each passing rate is summed over the path variables in the sensitive group. We also used Z3's in-built pseudo-Boolean functions \emph{PbLe} (\emph{PbGe}) and \emph{AtMost} (\emph{AtLeast}). These cardinality functions are optimised to handle linear arithmetic expressions with pseudo-Boolean variables. As an example, the majority vote on point $x$ in Section \ref{forest} can be expressed as
\[
\maj\limits_{1\leq i\leq n} T_i(x) = \mathrm{AtLeast}(\{T_i(x):1\leq i \leq n\}, \left \lceil{\dfrac{n}{2}}\right \rceil),
\]
where the right hand side expression returns a pseudo-Boolean value for if at least $\left \lceil{\dfrac{n}{2}}\right \rceil$ decision trees output \texttt{TRUE} on $x$. 

\section{Evaluation}
\label{evaluation}

FairRepair's goal is to repair an arbitrary decision tree to be fair while bounding the semantic change. In our evaluation, we study the following research questions.

\noindent \textbf{RQ 1:} \emph{Does our implementation of FairRepair produce a fair decision tree (random forest) given an unfair decision tree (random forest) input?} We run FairRepair with different fairness thresholds and observe the resulting fairness.
    
\noindent\textbf{RQ 2:} \emph{How do the semantics of the fair decision trees (random forest) produced by FairRepair compare to the input unfair decision trees (random forest)?} We compare the semantic difference between the input and output models produced by FairRepair.
    
\noindent\textbf{RQ 3:} \emph{How scalable is FairRepair?}  We evaluate FairRepair's runtime for different input datasets sizes.

\noindent\textbf{RQ 4:} \emph{How much does the decision model change in structure?} In particular, we measured the number of new paths produced during the refinement.

\noindent\textbf{RQ 5:}\emph{How much does the prediction accuracy change due to the repair?} We measured the precision and recall values to compare the accuracy of the models before and after repair.

\subsection{Methodology}

For our evaluation we use two datasets: the German Credit card dataset\footnote{\url{https://archive.ics.uci.edu/ml/datasets/statlog+(german+credit+data)}} and the Adult dataset\footnote{\url{https://archive.ics.uci.edu/ml/datasets/Adult}}, both from the UC Irvine Machine Learning Repository. These datasets have several sensitive attributes and come from domains in which fairness is important.

Table \ref{tab:datasets} shows the size and sensitive attributes of the datasets. The \emph{German} dataset classifies a person as a good or a bad credit risk. The dataset includes 1,000 persons' records and each record ranges over 20 features, such as the person's age and employment status.
The \emph{Adult} dataset is used to predict whether a person will earn over 50K dollars a year. The dataset has about 48K records and each record ranges over 14 features. We select attributes that are commonly recognised as sensitive features to form the sensitive groups. For both datasets, \emph{sex} is selected to be sensitive attribute. For German dataset, we selected whether a person is a \textit{foreigner worker} and for adult dataset, we selected \textit{race} as the second sensitive attribute.

\begin{table}
	\centering{}%
	\begin{tabular}{cc|cc|cc}
		\hline
		\multicolumn{6}{c}{\textbf{Adult Dataset}}\cr
		\hline
	    \multicolumn{2}{c|}{\textbf{Sex\footnotemark[4]}} & \multicolumn{2}{c|}{\textbf{Race}} & \multicolumn{2}{c}{\textbf{Sex and Race}}\tabularnewline 
		\emph{Min} & \emph{Max} & \emph{Min} & \emph{Max} & \emph{Min} & \emph{Max} \tabularnewline
		\hline 
		0.36 & 0.36 & 0.43 & 0.99 & 0.17 & 0.99\tabularnewline
		\hline 
		\multicolumn{6}{c}{}\cr
		\hline 
		\multicolumn{6}{c}{\textbf{German Dataset}}\cr
		\hline
	    \multicolumn{2}{c|}{\textbf{Sex}} & \multicolumn{2}{c|}{\textbf{F. Worker\footnotemark[4]}} & \multicolumn{2}{c}{\textbf{Sex and F. Worker}}\tabularnewline 
		\emph{Min} & \emph{Max} & \emph{Min} & \emph{Max} & \emph{Min} & \emph{Max} \tabularnewline
		\hline 
		0.74 & 0.94 & 0.45 & 0.45 & 0.12 & 0.99\tabularnewline
		\hline 
	\end{tabular}

	\caption{Min, max initial fairness (initial pairwise ratios of the passing rates across sensitive groups) in Adult and German dataset. 
	}
	\label{tab:init-fairness}
\end{table}

Tables~\ref{tab:init-fairness} list the minimum and maximum initial group fairness values (between 0 and 1) we evaluate in the Adult and German datasets. For each dataset, these fairness values are pairwise ratios of the initial passing rates across their sensitive groups. The (group) fairness inequalities defined in Section \ref{subsec: defs} require all these ratios to be greater than fairness threshold $c$ after the repair procedure.

\begin{table}[t]
\centering
{\small
\begin{tabular}{ r| rccp{2.5cm} }
 \hline
 \textbf{Dataset} & \textbf{Points} & \textbf{Classes} & \textbf{Features} & \textbf{Sensitive attributes} \\
 \hline
\emph{German} & 1,000 & 2 & 20 & Sex (4), Foreign-Worker (2), Sex and Foreign-Worker (8)\\
 \hline
 \emph{Adult} & 48,842 & 2 & 14 & Sex (2), Race (5), Sex and Race (10)\\
 \hline
 \end{tabular} 
}
\caption{Datasets used in our evaluation.
Numbers in parenthesis indicate the number of groups that must be fair relative to each other. For example, there are 4 groups, for the \emph{Sex} (\emph{"Personal status and sex"}) attribute in the German dataset. 
}
\label{tab:datasets}
\end{table}

\footnotetext[4]{With two groups, the min and max are just inverses of one another.}

All experiments were run with Python 3.5.2, on a Linux Ubuntu 16.04 server with 28-core 2.0 GHz Intel(R) Xeon(R) CPU and 62GB RAM.
When generating a scikit-learn \emph{DecisionTreeClassifier} and \emph{RandomForestClassifier}, we used an 80/20 training-test split of the dataset. In the experiments, we used the same datasets for training and repairing the models. However, we do not require them to be the same in general.

For both decision tree and random forest, we collect results on the two datasets respectively, over 6 fairness thresholds, 6 semantic bounds and 3 choices of sensitive attribute(s) (in total 108 sets of parameters). In the following, if not specified, the plots show runs on all varying parameters. For experiments on random forests, we select the default forest size to be 30 trees. We measured the classification accuracy for random forests of different sizes. Figure \ref{fig:num-of-trees} shows the plots for both datasets. The differences in classification error rate is less than 1\% for forests with more than 30 trees  We also collected results on different forest sizes for scalability and syntactic change analysis. To reduce contingencies in the results, each experiment was run with 5 different random seeds\footnote[5]{Randomness is used by scikit-learn's decision tree training procedure.}.

\begin{figure}[t]
    \centering
    \includegraphics[width=\columnwidth]{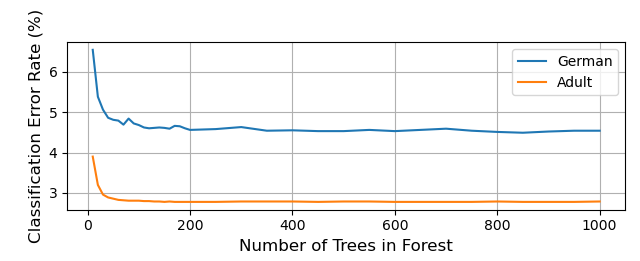}
    \caption{Classification Error Rate vs. Forest Size (Adult). 
    }
    \label{fig:num-of-trees}
\end{figure}

\textbf{RQ1: Fairness.} FairRepair must achieve the stated fairness guarantee. We used the two datasets to validate that our approach is able to achieve the fairness requirement. For decision trees, figures~\ref{fig:rq1-fairness-adult} and~\ref{fig:rq1-fairness-german} show scatter plots of the fairness achieved for different input fairness threshold values for multiple runs on the Adult and German datasets, respectively. 
The plots show results for repairs of all the sensitive attribute combinations in Table~\ref{tab:datasets}. 
For random forest, figure \ref{fig:rq1-forest-adult} shows data for all 540 FairRepair runs on the Adult dataset, with the tree size set to 30.
We do not show the corresponding plot for the German dataset (which is smaller), but the achieved output tree fairness is similar to Figure~\ref{fig:rq1-fairness-german}.

As expected, the fairness of the output model is bounded below by the fairness threshold (the $y=x$ line) and it is bounded above by the inverse of threshold ($y=1/x$ line). The reason for the inverse upper bound is because we show an achieved fairness for every two sensitive groups per experiment. For example, for experiments with \emph{Sex} as the sensitive attribute, we plot the fairness of the \emph{Sex=Female} group against the \emph{Sex=Male} group, and also the inverse fairness, as well. The fairness of all the points falls in the expected range, thus empirically validating our guarantee about fairness.

\begin{figure}[t]
    \centering
    \includegraphics[width=\columnwidth]{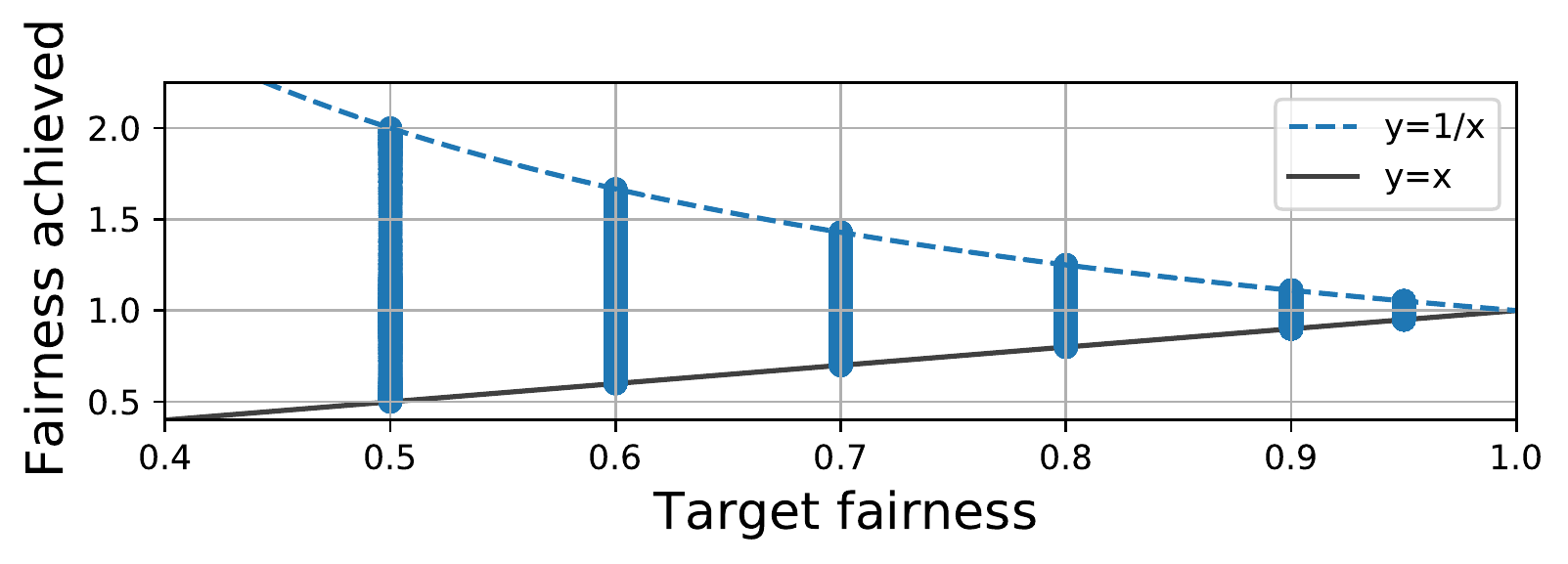}
    \caption{Achieved output fairness on the Adult dataset for different fairness thresholds (decision tree). 
    }
    \label{fig:rq1-fairness-adult}
\end{figure}
\begin{figure}[t]
    \centering
    \includegraphics[width=\columnwidth]{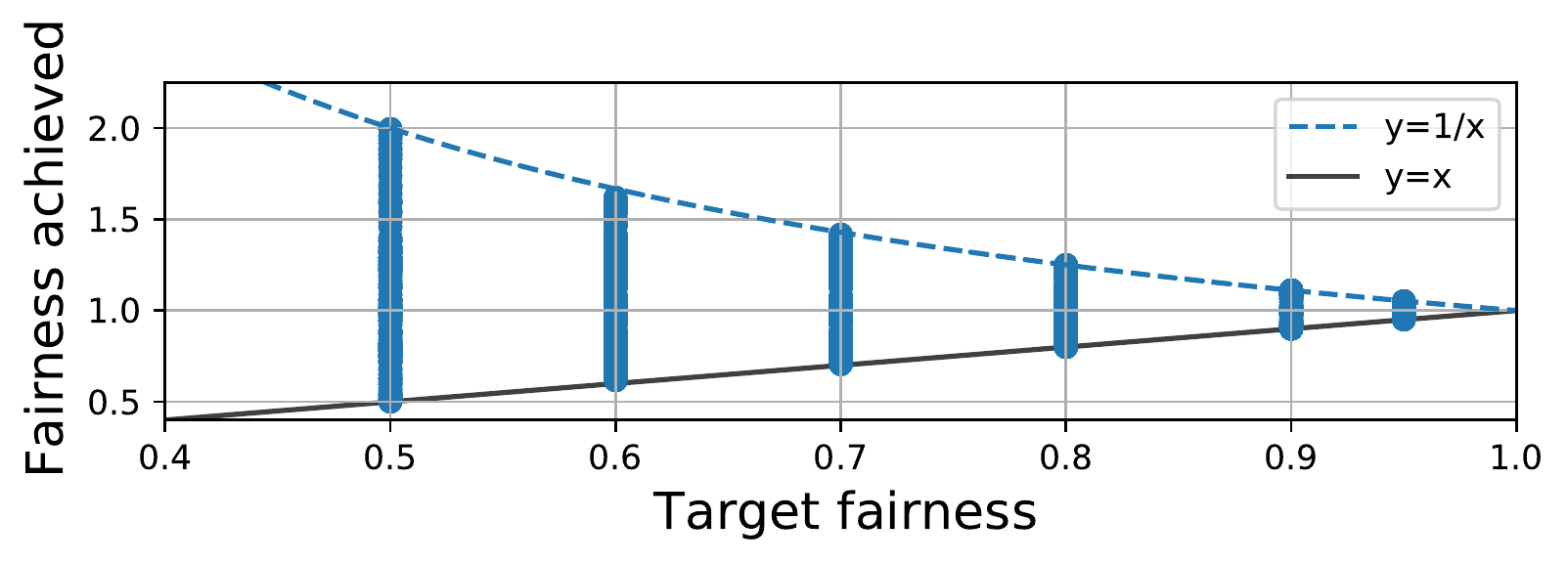}
    \caption{Achieved output fairness on the German dataset for different fairness thresholds (decision tree). 
    }
    \label{fig:rq1-fairness-german}
\end{figure}
\begin{figure}[t]
    \centering
    \includegraphics[width=\columnwidth]{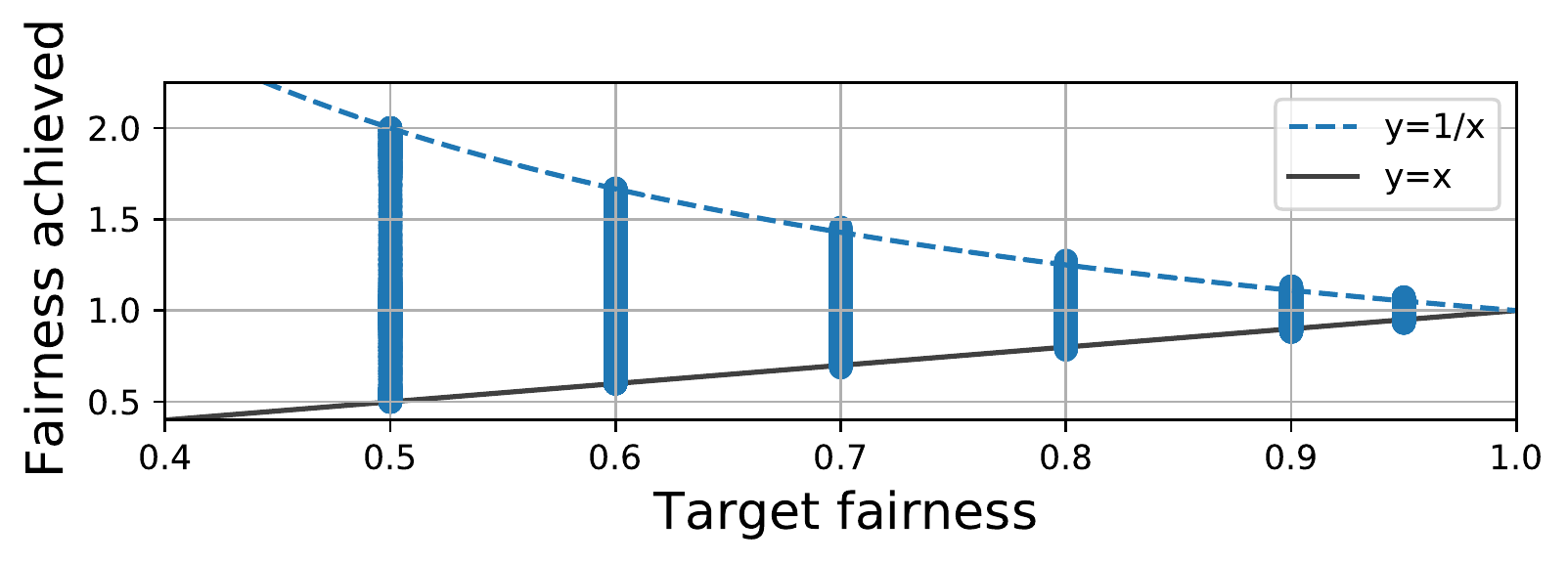}
    \caption{Achieved output fairness on the German dataset for different fairness thresholds (random forest). 
    }
    \label{fig:rq1-forest-adult}
\end{figure}

\textbf{RQ2: Semantic difference.} A fairness requirement could be trivially achieved by changing the classification output of path hypercubes in the tree that impact many data points. Achieving the minimal semantic difference (changing the classification output on the fewest number of points) is challenging. FairRepair makes a semantic difference guarantee -- given an $\alpha$ parameter, FairRepair will produce a repair that is within a $\alpha$ multiplicative constant of the minimal semantic repair. We collected results on various $\alpha$ values, ranging from 1.05 to 2.0. 
For decision trees, figure~\ref{fig:rq2:semantic-diff-adult} and~\ref{fig:rq2:semantic-diff-german} plot the semantic difference between the input and output decision trees for varying $\alpha$ and fixed $c=0.8$ for the Adult and German datasets, respectively. The $y$-axis shows the number of datapoints flipped during repair. Figure~\ref{fig:rq2:semantic-diff-adult} shows results for \emph{Sex} and \emph{Race} as the sensitive attributes. Figure~\ref{fig:rq2:semantic-diff-german} shows results for \emph{Sex} and \emph{Foreign Worker} as the sensitive attributes. 
FairRepair points in both figures fall into the region that is bounded below by the theoretical minimum of the semantic difference, and is bounded above by the multiplicative constant $\alpha$. We also checked that these results hold for other fairness thresholds and sensitive attributes.

The result for random forest is similar. Figure \ref{fig:RQ2-forest-adult} shows results for the Adult dataset with $c=0.8$, \emph{Sex} and \emph{Race} as the sensitive attributes. The semantic difference (in percentage) drops from 10\% to 5\% as $\alpha$ decreases. All semantic difference are bounded above by $\alpha$ for all settings of fairness thresholds and sensitive attributes.

\begin{figure}[t]
    \centering
    \includegraphics[width=0.975\columnwidth]{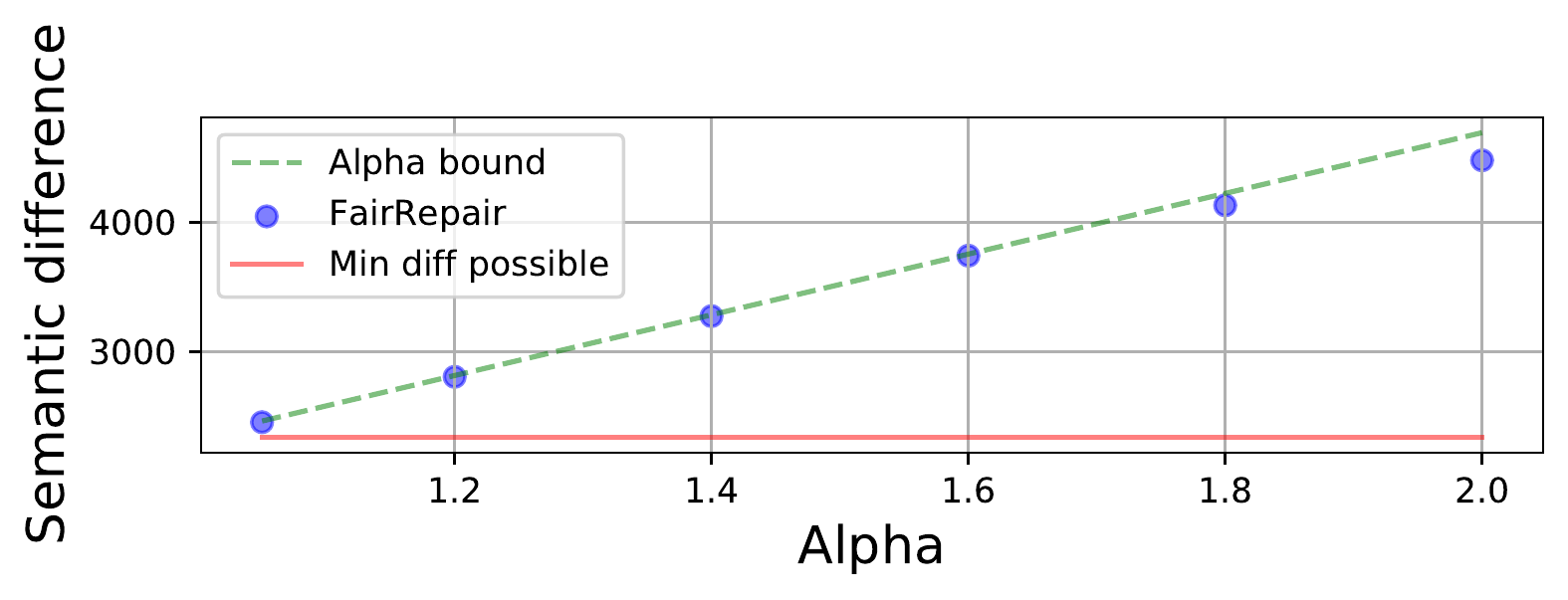}  
    \caption{The semantic difference between input and output random forests for different $\alpha$ bounds, in the Adult dataset. 
    }
    \label{fig:rq2:semantic-diff-adult}
\end{figure}

\begin{figure}[t]
    \centering
    \includegraphics[width=\columnwidth]{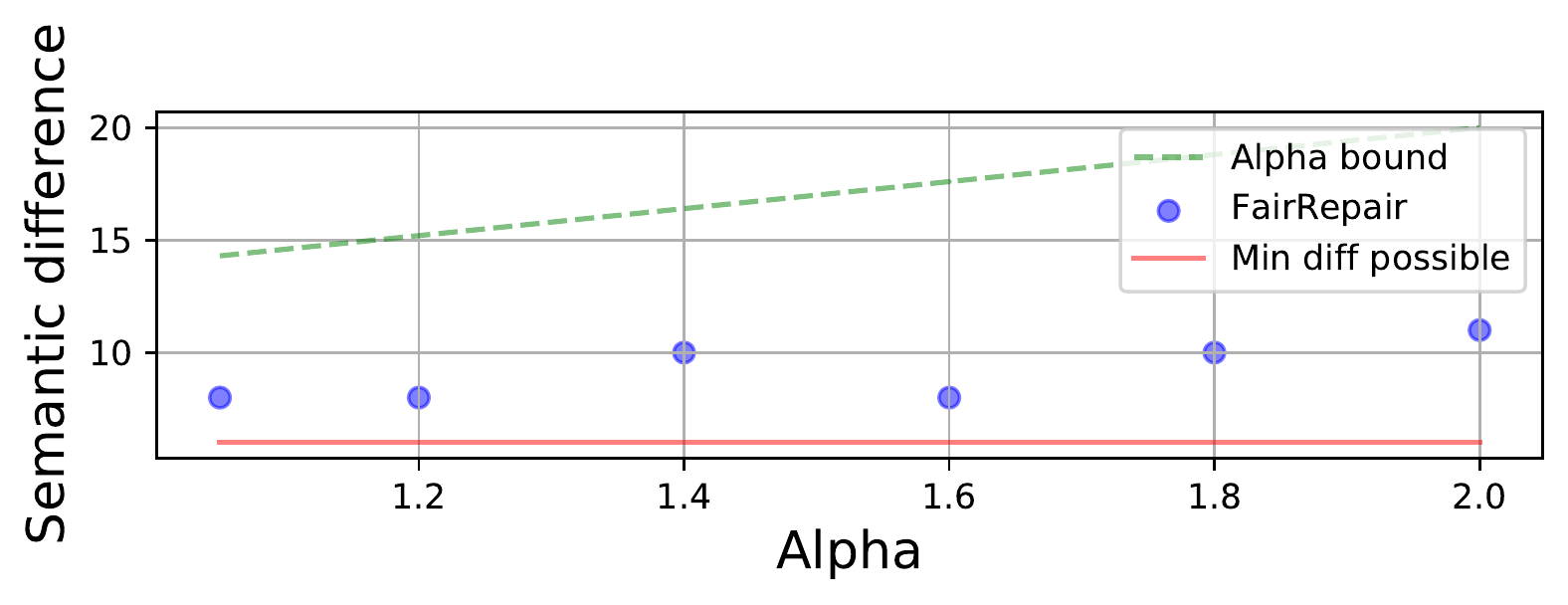}
    \caption{The semantic difference between input and output decision trees for different $\alpha$ bounds, in the German dataset. 
    }
    \label{fig:rq2:semantic-diff-german}
\end{figure}

\begin{figure}[t]
    \centering
    \includegraphics[width=\columnwidth]{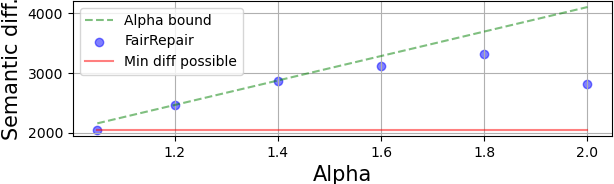}  
    \caption{The semantic difference between input and output random forests for different $\alpha$ bounds, in the Adult dataset. 
    }
    \label{fig:RQ2-forest-adult}
\end{figure}

\textbf{RQ3: Scalability in size of dataset.} FairRepair uses an input dataset to compute the fairness and semantic differences as it executes. For large datasets, these operations may be expensive. To investigate scalability we only use the Adult dataset since the German dataset is relatively small. Figure~\ref{fig:rq3-dataset-size} plots the total running time of FairRepair on datasets of varying sizes for decision trees. For this, we created partitions of the Adult dataset by sampling uniform random subsets of increasing size (from 5K to 45K points). The trained decision trees for these datasets varied in size, from 577 path hypercubes to 5,521 path hypercubes. In the figure we also label the runs according to their fairness threshold. We found that the choice of fairness threshold and $\alpha$ made little difference to the running time. There is no clear relationship between the fairness threshold and $\alpha$, and time to repair. These experiments used \emph{Sex} as the sensitive attribute.
The figure illustrates that larger datasets take longer to repair. For example, on an input of 45,000 points and a tree with about 5,000 path hypercubes FairRepair completes the repair in under 5 minutes.

For random forest, the running time is shown in figure \ref{fig:rq3-forest}. Since fairness and $\alpha$ does not significantly affect the running time, we fixed fairness threshold to be 0.8 and $\alpha$ to be 1.2. As the size of the dataset increases, it takes up to 30 minutes (on average) to repair a biased random forest on the full Adult dataset. 

We also measured the scalability against various forest size. Figure \ref{fig:rq3-forest-size} shows the running time for various forest sizes ranging from 10 trees to 100 trees, with fixed fairness threshold 0.8 and $\alpha=1.2$. We observe a proportional increase in the amount of repair time as the forest size increases. For random forest on the Adult dataset with 100 trees, FairRepair takes up to 90 minutes to find a repair. 

We also note that our code is in Python and we believe it could be further optimized to achieve much better scalability.

\begin{figure}[t]
    \centering
    \includegraphics[width=\columnwidth]{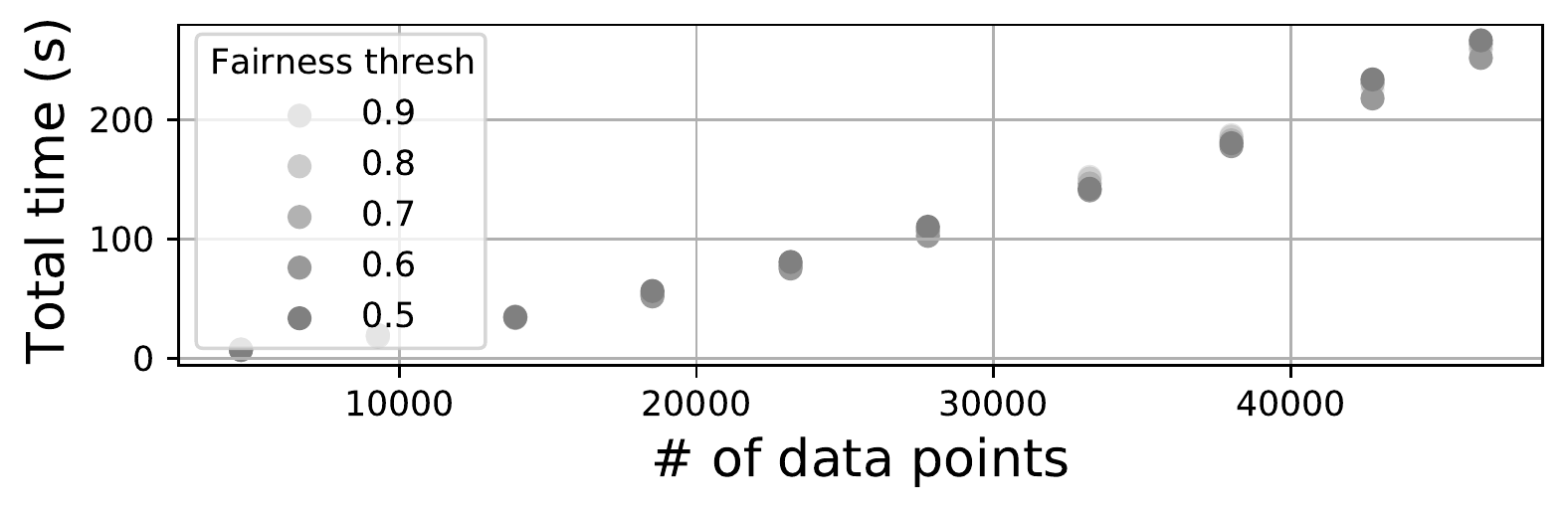}
    \caption{The total running time of FairRepair for datasets with increasing number of points and different sensitive groups (decision tree, adult dataset).}
    \label{fig:rq3-dataset-size}
\end{figure}

\begin{figure}[t]
    \centering
    \includegraphics[width=\columnwidth]{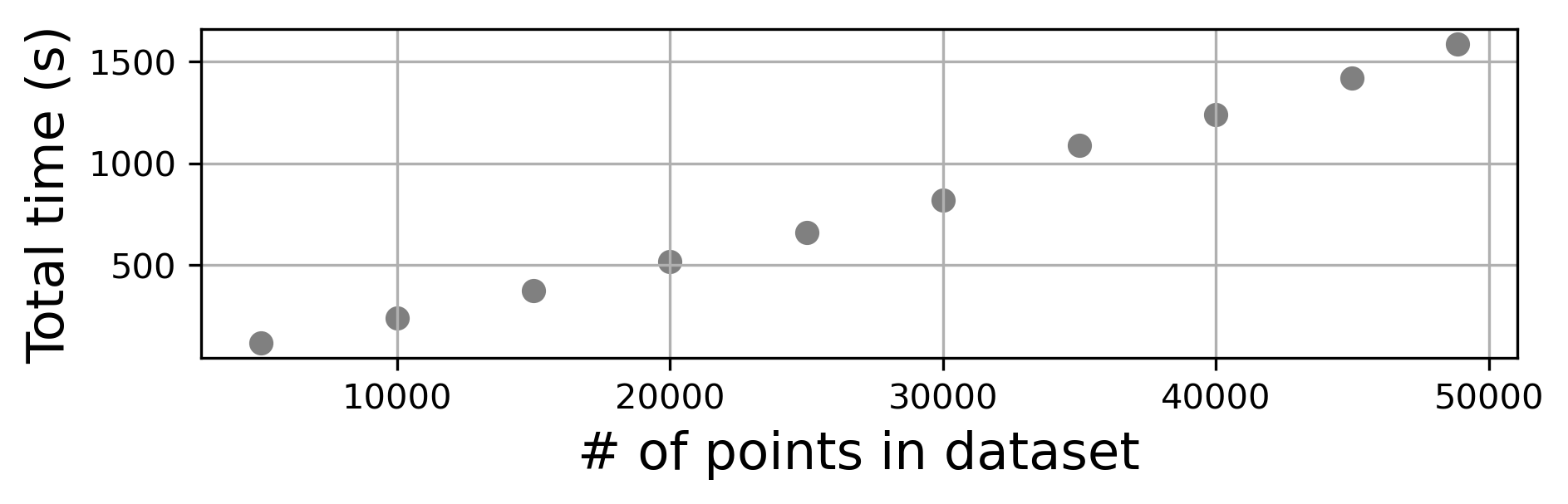}
    \caption{The total running time of FairRepair for datasets with increasing number of points and different fairness thresholds (random forest, adult dataset).}
    \label{fig:rq3-forest}
\end{figure}
\begin{figure}[t]
    \centering
    \includegraphics[width=0.95\columnwidth]{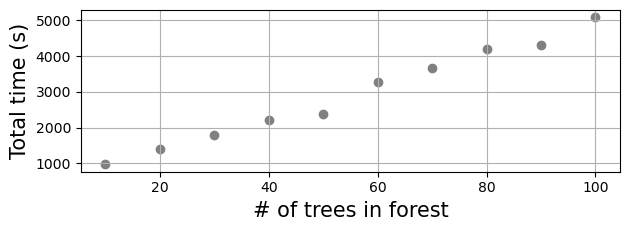}
    \caption{The total running time of FairRepair for random forests on the Adult dataset with increasing forest size.}
    \label{fig:rq3-forest-size}
\end{figure}

\textbf{RQ4: Syntactic Difference.}
In our algorithm, no path hypercubes are removed or merged during the repair. Hence the syntactic change can simply be defined as the (average) increase in the number of path hypercubes of the decision tree (or over the decision trees in the forest). In the results of decision trees for both datasets, we observe that no runs require refinement, i.e., the number of path hypercubes is unchanged. Hence we only discuss the results on random forests. We fixed the fairness threshold to be 0.8. The results for syntactic change against initial number of path hypercubes are shown in Figure  \ref{fig:rq4-num-hcubes}. As the initial number of path hypercubes increases, the syntactic change increases proportionally. For a random forest with average 30K path hypercubes in each decision tree, our repair increases the average number by less than 1,400
(less than ~5\% change).


\begin{figure}[t]
    \centering

    \includegraphics[width=\columnwidth]{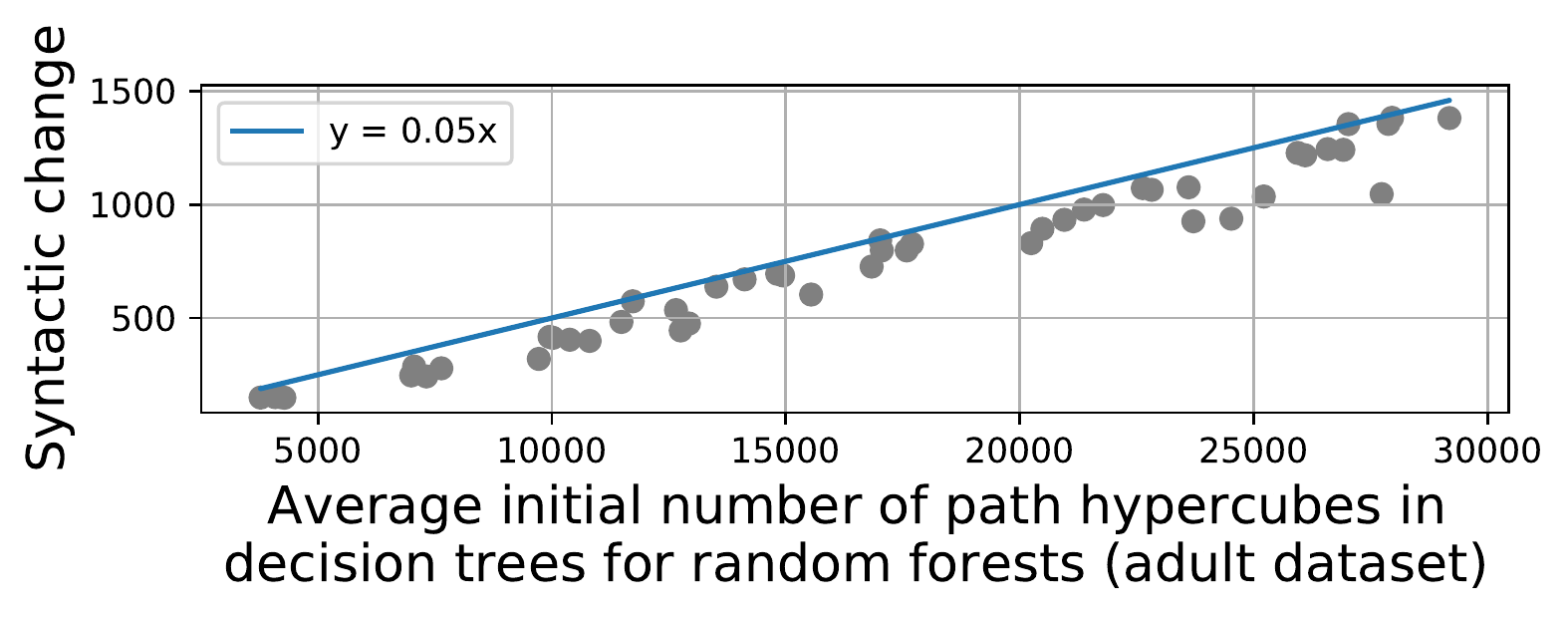}
    \caption{Syntactic change due to FairRepair for input random forests}
    \label{fig:rq4-num-hcubes}
\end{figure}


\textbf{RQ5: Accuracy.}
While achieving the fairness requirement, the change in the classification accuracy of the repaired models is limited. Figure \ref{fig:rq5-adult} shows the classification accuracy against various fairness thresholds on decision trees trained on the Adult dataset, with $\alpha=1.05$ and sensitive attribute set to \textit{sex} and \textit{race} respectively. 
The classification accuracy before repair was (on average) 96.3\% over 5 decision trees trained with different random seeds. %
For sensitive attribute \textit{race}, the accuracy at $c=0.95$ is 94.8\%, which only decreases by 1.5\% after the repair.
For sensitive attribute \textit{sex}, the accuracy at $c=0.95$ is 90.5\%. The initial passing rates computed from the predicted results are (on average) 31.0\% and 11.3\%. The initial model has a high accuracy, hence the unfairness in the dataset is highly preserved in the initial model. The passing rates computed directly from the Adult dataset are 30.9\% for male and 11.2\% for female, which is only one-third of males. To balance the passing rates, in particular, to increase the passing rate for females (and to decrease the passing rate for males), some female applicants (data points) who were initially classified with a $\mathtt{FALSE}$ prediction have to be given $\mathtt{TRUE}$ prediction now (and some male applicants who were initially classified with $\mathtt{TRUE}$ prediction are now given $\mathtt{FALSE}$ prediction). This causes the loss in accuracy. However, such loss is still bounded as shown by our results.
FairRepair is able to repair the decision trees, bounding the ratio of the passing rates within a fairness threshold of $c = 0.95$, while maintaining a 90.5\% accuracy.  The final passing rates are 30.9\% and 29.3\% for male and female applicants respectively.  

The results for decision trees on German dataset is shown in Figure \ref{fig:rq5-german}, where for sensitive attribute \textit{foreign worker}, the accuracy is 91.7\%, and for sensitive attribute \textit{sex and marital status}, the accuracy is 93.3\% at fairness threshold 0.95, compared to the initial accuracy 93.9\%. Such results show that besides keeping a small semantic distance from the original decision model, our algorithm is able to produce a repair that does not damage the classification accuracy. Results on random forests show a great similarity to Figure \ref{fig:rq5-adult} and \ref{fig:rq5-german}, hence are not displayed.

\begin{figure}[t]
    \centering
    \includegraphics[width=0.95\columnwidth]{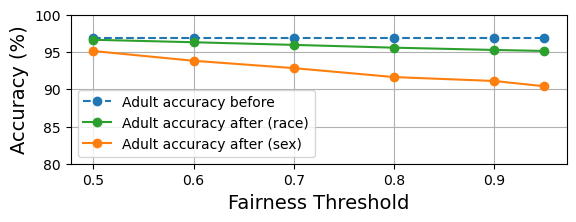}
    \caption{Classification accuracy of decision trees on Adult dataset on two sensitive attributes}
    \label{fig:rq5-adult}
\end{figure}

\begin{figure}[t]
    \centering
    \includegraphics[width=1.06\columnwidth]{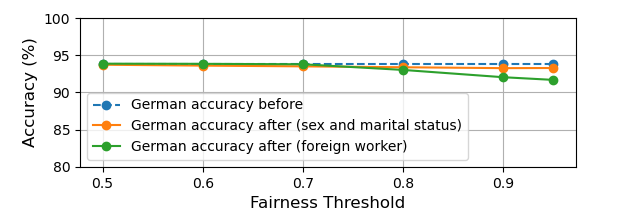}
    \caption{Classification accuracy of decision trees on German dataset on two sensitive attributes}
    \label{fig:rq5-german}
\end{figure}

\section{Related Work}
\label{related}

\paragraph{Fairness and its Testing / Analysis}
There has been an increased research effort on fairness in machine learning in recent years. Various fairness definitions have been proposed~ \cite{calders09, hardt16, kilbertus17, kusner17}. In our paper, we used group fairness, which is used in the fairness testing work of \cite{FairnessTestingESECFSE2017}. Other recent works on fairness testing includes black-box fairness testing~\cite{FSE19} and individual fairness testing~\cite{ASE18}.
There have also been attempts to verify fairness properties, including calculating integration with measure theoretical results \cite{OOPSLA17}, and using concentration inequalities to bound the error \cite{OOPSLA19}. Many approaches to train fairness classifiers have also been studied \cite{davies17, dwork12, fish16}, including modifying dataset fragments to remove sensitive attribute correlations \cite{fledman14}, and encoding fairness requirements into classifier training process \cite{hardt16}.
Fairness is a special form of probabilistic property. The study of probabilistic properties involves software engineering techniques, as studied in past works, including probabilistic symbolic execution \cite{geldenhuys12}\cite{luckow14}\cite{fliliri13}, probabilistic model checking \cite{vitaly03}\cite{kwiat10}\cite{clarke11}, and probabilistic program analysis / verification \cite{OOPSLA17}\cite{sampsons14}\cite{OOPSLA19}.

\paragraph{Automated Program Repair}
There is a large literature in the field of  automated program repair (e.g. \cite{CACM19} is a recent review of the research area). Automated repair techniques aim to produce patches to meet certain specifications such as passing a given test-suite. Heuristic repair, or search-based repair, iterates over a predefined search space of fixes, and checks the validity of the fixes (e.g. see \cite{weimer09}). To introduce higher flexibility than simple mutations, heuristics like plastic surgery hypothesis \cite{marks10} or competent programmer hypothesis are used to assume the patch can be produced from (i) other parts of the code \cite{ray16}, or (ii) guided by templates obtained from human patches \cite{kim13}, or (iii) via the use of learning to prioritize patch candidates \cite{prophet}.
In contrast to search-based or template-based works on program repair, our work is more related to constraint-based or semantic repair techniques \cite{hoang13,nopol,sergey16}. In these techniques a repair constraint is inferred as a specification via symbolic execution, and this specification is used to drive a program synthesis step which generates the patch. Previous works on semantic repair are mostly test-driven, where the repair attempts to pass a given test-suite.  In contrast, fairness is a hyperproperty \cite{hyperproperty} which is an accumulative property over collection of traces, and thus fairness repair is a fundamentally different problem from repairing a program based on given tests/assertions. 

\paragraph{Fairness in Decision Tree Learning}
Other than repairing a trained decision tree or random forest, attempts have been made to integrate fairness into the model training procedure \cite{aghaei}\cite{friedler18}. Among them, Kimaran et al. \cite{kamiran} also employ the technique of relabeling leaf nodes in decision trees to achieve fairness requirements. They adopted a different fairness definition that requires the difference (instead of the ratio) of the passing rates to be bounded. They compute the contribution to minimising the difference of flipping each leaf and use a greedy algorithm to flip the leaves accumulatively. Such definition restricts their algorithm to be applied to binary sensitive attributes only. Aghaei et al. \cite{aghaei} attempted to construct an optimization framework for learning optimal and fair decision trees, which aims to mitigate unfairness caused by both biased dataset and machine mis-classification. However, these works do not provide a solution for repairing an already trained unfair decision tree.

\paragraph{DIGITS \cite{CAV17}}
The work of \cite{CAV17} performs fairness repair guided by input population distributions. In contrast, our work is focused on rectifying an unfair decision tree from an unfair dataset given as the input. A dataset appear to be similar to an input probability distribution, but they represents two different methodologies. Our approach is able to cure unconscious biases in a decision making entity. With a dataset as input, we first construct a decision tree/random forest with a good fit, then produce a minimal repair to the model. The original models should closely capture the biases in the dataset. This allows us to detect and to cure the implicit and unconscious biases in the previous decision making procedure. The repaired model could be used as a guideline for making future decisions, and it also works as an explanation for the previous unfairness. In contrast, DIGITS repairs a model that already exists. The model is repaired regardless of its accuracy or whether it captures the original biases in the training dataset. Assuming the model closely resemble the original biases, the DIGITS method requires continuing sampling data points from the input distribution to achieve better convergence to the optimal repair. This requires more information of the population, compared to just the past decision results.

The completeness guarantee of DIGITS is relative to assumptions such as a manually provided sketch (describing repair model), which is not needed in FairRepair. While its algorithm is designed for binary sensitive attributes, ours is able to handle multiple sensitive groups. Our results are not directly comparable with DIGITS, although they adopted similar fairness criteria and bench-marked on the Adult dataset. Their decision trees were relatively small (less than 100 lines code) and employed at most three features. In contrast, our decision trees employ at least 10 features and scale to real-life applications, with 10K leaves (path hypercubes) before splits on sensitive groups. With sensitive attribute being \emph{sex} and fairness threshold set to 0.85, their algorithm was able to produce a repair in 10 minutes with semantic difference of 9.8\%. In contrast, FairRepair is able to achieve a semantic difference of 6.2\% with a higher fairness threshold 0.9 within 5 minutes. According to DIGITS synthesis algorithm, longer training time allows better convergence to the optimal repair, so it might find a repair that satisfies the fairness requirements at an early stage but far from the optimal. But the comparison does show the efficiency of our tool on both achieving fairness criteria and bounding semantic difference.


\section{Discussion}
\label{discussion}

 In this paper we described FairRepair, a process to repair an unfair decision tree to be fair relative to a dataset. Compared to existing approaches, FairRepair is more data-driven, does not require population distributions,  and makes fewer assumptions about the input dataset and the decision tree. At the same time, FairRepair provides useful guarantees about relative  completeness and semantic difference. We provide proofs of these guarantees, validate them experimentally, and evaluate FairRepair on two datasets that are used in the fairness literature. Our tool FairRepair has been submitted as {\em confidential supplementary material}, along with this paper submission. The tool will be released open-source for the wider research community, if and when our paper is published.

\paragraph{Decision Tree Learning Algorithm}
In this paper, we used scikit-learn \textit{DecisionTreeClassifier} to train the decision trees and random forests. However, our algorithm is generic and not limited to any particular tree learning algorithms. We extract the path hypercubes of the decision trees at the very beginning of the repair procedure, and did not use the tree structure in the later steps. Since the decision tree paths and path hypercubes have a one-to-one correspondence, there is not 
restriction on the choice of the tree classifier.

\paragraph{Accuracy of Models}
The main task of our algorithm is to repair a decision tree or a random forest such that it satisfies a probabilistic group fairness requirement. In this procedure, accuracy of the models would be affected, since we modified the classification of the nodes. It is desirable to have a repaired model that is also accurate. However, it is not our first priority to maintain the accuracy, but to repair the model with minimal change. 
Instead, we proposed the notion of semantic change to measure the difference between the original and the repair model. Similar to maximising accuracy, our algorithm aims to bound the semantic difference. 

There are accuracy-based fairness definitions, like \textit{equal opportunity} which requires the false negative rates of the sensitive groups to be similar, and \textit{overall accuracy equality} which requires the accuracy over all sensitive groups to be similar. Such fairness notions should be considered independently from group fairness.

\paragraph{Beyond Group Fairness}
Our approach focuses on group fairness~\cite{calders09}, which requires members of the sensitive groups to be classified at approximately the same rate. There are, however, many other fairness notions such as \emph{causal fairness} or 
\emph{equal opportunity fairness}~\cite{hardt16}. We provide the full algorithms including pre-processing and SMT encoding of repairing decision trees with respect to these two fairness definitions in Appendix \ref{other-defs}.

One can thus look into domain specific language (DSL) designs that can capture different notions of fairness, and integrate them with the repair process in FairRepair.

\paragraph{Qualification Fairness}
In our algorithm, we repaired the models with respect to group fairness (or statistical parity), which requires the passing rates of each sensitive group to be similar, without considering the features of the datapoints that received the desired outcome. In \textit{conditional statistical parity}, the passing rates are computed over \textit{controlled} subgroups of the sensitive groups. One typical control is \textit{qualification}. In our loan example, \textit{qualified group fairness} would require the proportions of candidates that received a loan among \textit{qualified} males and females be similar. Our problem formulation does not consider qualification, since this requires specifying the attributes that are legitimate to be used in making decisions. This is not the main focus of our research, which makes it possible that more qualified applicants do not receive the desired outcome, while the less qualified do.

\paragraph{Decision Trees and Beyond}
We believe that our repair approach is applicable to other types of models whose behavior on an input can be symbolically expressed in a closed form, such as a set of constraints for a decision tree, random forest, or an SVM model. We implemented our approach, however, only for decision trees. Additionally, by relying on sci-kit learn, we adopted several limitations, such as the need for a OneHot encoding of categorical attributes. We hope to explore other model variants in our future work.
One interesting feature of FairRepair is that it explicitly uses the sensitive attributes in the output decision tree. This is partly a dependency of our approach on decision trees, which needs to be investigated further for other models.



\paragraph{Beyond fairness.} Fairness is an example of a hyper-property over collections of traces \cite{hyperproperty}. There are other types of hyper-properties of interest such as robustness, side-channel, non-interference. These can all be expressed as hyper-properties over a set of program traces. At a high level, our approach indicates the promise of hyper-property driven automated repair. This raises the prospect of future generation program repair systems being systematically supported by SMT solver back-ends with a hyper-property plugin. 



\begin{appendices}

\section{Other Fairness Definitions}
\label{other-defs}
The algorithm in Section \ref{main algo} explains how we can repair a decision tree or a random forest with respect to group fairness, or statistical parity. There are some other fairness notions raised in the community. In particular, we are interested in \textit{fairness definitions based on similarity}, and \textit{fairness definitions based on precision}. 

\subsection{Causal Discrimination}
We shall still explain the definitions in the context of a binary decision making problem. One representative of similarity-based fairness measure is \textit{causal fairness}, which is adapted from \textit{causal discrimination score} defined in \cite{FairnessTestingESECFSE2017}. It requires applicants who have similar attributes to receive the same outcome. In particular, if two applicants only differ in the sensitive attributes, they should both be classified as 0 or 1. If they receive different outcomes, we refer them as experienced \textit{individual unfairness}. The causal discrimination score $S$ is measured by the size of the proportion (as a ratio between 0 and 1) among input population, which contains individuals that by only changing its sensitive attributes, the decision outcome also changes. A \textit{fairness threshold} $c$ for causal fairness is a real number between 0 and 1, such that the decision making model is called \textit{causally fair} if $S \leq c$. 

To repair a decision tree $T$ with respect to causal fairness, we start with pre-processing $T$. We first collect the path hypercubes of $T$. Let $\mathcal{H}:=\{H_1,\dots,H_k\}$ denote the set of path hypercubes of $T$, and let $\mathcal{D}$ denote the input dataset. Note that $\mathcal{H}$ is a partition of the input space. For each datapoint $d\in\mathcal{D}$, there exists a unique $H_i\in\mathcal{H}$ such that $d\in H_i$. Let $\mathcal{M}:\mathcal{D} \to \{1,2,\dots,k\}$ be a mapping such that $H_{\mathcal{M}(d)}$ is this unique path  hypercube. We use $H_{\mathcal{M}(d)}(d)$ to denote the predicted outcome of $d$, i.e., $H_{\mathcal{M}(d)}(d) = T(d)$. Let $\mathcal{A}$ be the sensitive attribute. There can be multiple sensitive attributes. Here we use a single sensitive attribute for illustration. Note that we force $\mathcal{A}$ to be categorical. Let $M:=\{A_1,\dots,A_m\}$ be the set of sensitive groups, where each $A_i$ is a valuation of $\mathcal{A}$. We group datapoints in $\mathcal{D}$ based on their non-sensitive attributes. Two datapoints $d_1$ and $d_2$ are in the same group $D$ if and only if they differ only in the sensitive attribute. We abuse the notation by letting $\mathcal{D}=\{D_1,\dots,D_l\}$. For each $D\in \mathcal{D}$, if 
\begin{align*}
   \mathcal{F}(D) &= \bigvee\limits_{d,d^{'}\in D} H_{\mathcal{M}(d)}(d)\neq H_{\mathcal{M}(d^{'})}(d^{'}) \\
   &= \exists d,d^{'}\in D.  H_{\mathcal{M}(d)}(d)\neq H_{\mathcal{M}(d^{'})}(d^{'}) 
\end{align*}
holds, then $D$ contributes to the causal discrimination. In particular, \[S = \dfrac{\sum\limits_{D\in\mathcal{D},D\models \mathcal{F}}\left\lvert D\right\rvert}{\left\lvert\mathcal{D}\right\rvert}.\]

We next construct the SMT formulas. For each path hypercube $H_i$, assign a pseudo-Boolean SMT variable $X_i$ ($X_i\in\{0,1\}$, representing a Boolean value). For points $d \in \mathcal{D}$, we use $X_{\mathcal{M}(d)}$ to denote the Boolean variables assigned to path hypercube $H_{\mathcal{M}(d)}$. Note that $X_{\mathcal{M}(d)}$ and $X_{\mathcal{M}(d^{'})}$ may refer to the same Boolean variable $X_i$ for $d\neq d^{'}$ if $\mathcal{M}(d)=\mathcal{M}(d^{'}) = i$. Substituting $X_i$'s into  expression $\mathcal{F}(D)$, we get $\bigvee\limits_{d,d^{'}\in D} X_{\mathcal{M}(d)}\neq X_{\mathcal{M}(d^{'})}$, which represents that at least two points $d,d^{'} \in D$ have different outcomes. Thus the fairness requirement is formulated as:
\[
\dfrac{\sum\limits_{D\in\mathcal{D}}\bigvee\limits_{d,d^{'}\in D} X_{\mathcal{M}(d)}\neq X_{\mathcal{M}(d^{'})}\cdot\left\lvert D\right\rvert}{\left\lvert\mathcal{D}\right\rvert} \leq c,
\]
where the numerator of the left hand side represents datapoints that experience individual unfairness. By dividing the size of the dataset, the fraction must be bounded by the fairness threshold $c$.

Similar to group fairness, we want to account for the semantic change as well. We first compute the initial causal discrimination score $S$. If $S>c$, then the minimal (theoretical) semantic change is $S-c$. This value is the minimal proportion in the input population that have to be assigned an opposite outcome to achieve causal fairness. (If $S\leq c$ then the decision tree is causally fair and needs no repair.) We abuse the notation $|H_i|$ to represent the number of datapoints reside in $H_i$, and use $H_i$ to also represent the initial return value of the hypercube. With $\alpha$ being the semantic difference bound, the semantic change is encoded as follows:
\[
\dfrac{\sum\limits_{H_i\in\mathcal{H}}\mathrm{Xor}(H_i,X_i)\cdot\left\lvert H_i\right\rvert}{\left\lvert\mathcal{D}\right\rvert} \leq \alpha\cdot(S-c)
\]

\subsection{Fairness by comparing predicted and actual outcomes}

The fairness definitions listed in this section compare the predicted outcome and actual outcome of datapoints. Before we proceed to the definitions, we denote $T(d)$ as the predicted outcome of datapoint $d$, and denote $Y(d)$ as the actual outcome of $d$. We shall omit $d$ and only write $T$ and $Y$ when $d$ is clear or not specified. For illustration, we consider the sensitive groups to be $M=\{\mathrm{male},\mathrm{female}\}$. In general, $M=\{A_1,\dots,A_m\}$, where $m$ is the number of sensitive groups in the decision making scenario.

The definitions in this section uses statistical metrics like false negative and positive predictive value. We assume the readers are familiar with these terms. 

\begin{definition}[Equal Opportunity]
A decision making model satisfies \textit{equal opportunity} if each sensitive group has similar false negative rate. In particular, it requires that $P(T=0|Y=1,\mathrm{male})=P(T=0|Y=1,\mathrm{female})$. The fractions of input population having actual outcome 1 that are classified to be 0 should be similar among each sensitive groups. 
\end{definition}

In practice, it is unlikely that the false negative rates of each sensitive group are all equal. Hence we define a fairness threshold $c$ to be a real value between 0 and 1, such that if for any two sensitive groups $A_i$ and $A_j$, $P(T=0|Y=1,A_i)\geq c\cdot P(T=0|Y=1,A_j)$, then the decision making model satisfies equal opportunity.

Note that the measurement of fairness for equal opportunity is similar to group fairness, with passing rates replaced with false negative rates. We shall not go through the detailed pre-processing steps again. Assume we have obtained the path hypercubes split based on the sensitive groups. Let $S_i, 1\leq i\leq m$ be the set of path hypercubes for each sensitive group $A_i$. Each $S_i$ contains path hypercubes $H_{i,j}$'s where $j$ indexes over the path hypercubes in the sensitive group $A_i$. The setting is the same as that of group fairness. 

To repair a decision tree $T$ with respect to predictive parity, we first compute the false negative rates for each sensitive group, and compute the minimal (theoretical) semantic change required to repair the tree. For sensitive group $A_i$, the false negative rate $F_i$ is as follows:
\begin{align*}
F_i &= P(T=0|Y=1,A_i) \\
&= \frac{\left\lvert\{d\in\mathcal{D}:d \text{ is in } A_i, T(d) = 0 \text{ and }Y(d) = 1\}\right\rvert}{\left\lvert\{d\in\mathcal{D}:d \text{ is in } A_i \text{ and } Y(d) = 1\}\right\rvert}.    
\end{align*}

Let $r_i$ denote $\left\lvert\{d\in\mathcal{D}:d \text{ is in } A_i, T(d) = Y(d) = 1\}\right\rvert$. We want to compute $x_i$ for each $r_i$ that is the false negative rate after repair We construct the following linear optimisation problem to compute the minimal semantic change, where $m$ is the number of sensitive groups. 
\[
\begin{cases} 
\forall \,1 \leq i,j\leq m, F_i = \dfrac{x_i}{\left\lvert\{d\in\mathcal{D}:d \text{ is in } A_i \text{ and } Y(d) = 1\}\right\rvert},\\
\forall \,1 \leq i,j\leq m, \dfrac{F_i }{F_j} \geq c , \\
\forall \,1 \leq i\leq m, 0 \leq F_i \leq 1, \\
\text{minimise } \sum\limits_{i=1}^{M}|x_i-r_i|.
\end{cases}
\]

We next proceed to construct the SMT formulas for fairness and semantic requirement. For simplicity, denote 
\[D_i:=\left\lvert\{d\in\mathcal{D}:d \text{ is in } A_i \text{ and } Y(d) = 1\}\right\rvert\]
to be the set of datapoints in sensitive group $A_i$ that have outcome 1. Similar to that in causal fairness repair, let $\mathcal{M}_i:D_i\to \{1,\dots,|S_i|\}$ be the function that maps each point $d\in D_i$ to the index of the unique path hypercube it resides in, i.e., $d\in H_{i,\mathcal{M}_i(d)}$. Assign a pseudo-Boolean SMT variable $X_{i,j}$ to each $H_{i,j}$. We abuse the notation to also allow $H_{i,j}$ to represent the Boolean outcome of the path hypercube $H_{i,j}$.

The false negative rate $F_i$ for $A_i$ is represented by:
\[
F_i = \frac{\sum\limits_{d\in D_i}\neg X_{i,j}}{|D_i|}.
\]
Note that, $X_{i,j}$'s are pseudo-Boolean and $\sum\limits_{d\in D_i}\neg X_{i,j}$ is a non-negative integer. The fairness requirement is thus encoded as:
\[
\forall 1\leq i,j\leq m, F_i \geq F_j.
\]
The semantic difference requirement is as follows:
\[
\sum\limits_{1\leq i\leq m}\sum\limits_{d\in D_i}\mathrm{Xor}(X_{i,\mathcal{M}(d)},H_{i,\mathcal{M}(d)})\leq \alpha\cdot\sum\limits_{i=1}^{M}|x_i-r_i|.
\]

The above is only one of the fairness definitions based on precision. Here we do not list all of them. The readers may find more for their own interests.






\end{appendices}

\bibliographystyle{acm}
\bibliography{references}

\end{document}